\documentclass{article} 
\usepackage[final]{colm2026_conference}

\usepackage{microtype}
\usepackage{hyperref}
\usepackage{url}
\usepackage{booktabs}

\usepackage{lineno}

\definecolor{darkblue}{rgb}{0, 0, 0.5}
\hypersetup{colorlinks=true, citecolor=darkblue, linkcolor=darkblue, urlcolor=darkblue}

\title{PerMix-RLVR: Preserving Persona Expressivity under\\Verifiable-Reward Alignment}

\author{
Jihwan Oh$^{1}$\thanks{Equal contributions}\quad
Soowon Oh$^{1\,2}$\footnotemark[1]\quad
Murad Aghazada$^{1}$\quad
Minchan Jeong$^{1}$\quad \\[0.2em]
\textbf{MyeongSeok Kang}$^{3}$\quad 
\textbf{Sungnyun Kim}$^{1}$\thanks{Corresponding authors}\quad
\textbf{Se-Young Yun}$^{1}$\footnotemark[2]
\\[0.5em]
$^{1}$KAIST AI\quad
$^{2}$Samsung Advanced Institute of Technology\quad
$^{3}$Seoul National University\quad
\\[0.25em]
}

\usepackage{wrapfig}
\usepackage{amsmath, amssymb, amsthm, bm}
\usepackage{graphicx}
\usepackage{subcaption}
\usepackage{multirow}
\usepackage{array}
\usepackage{tabularx}
\usepackage{enumitem}
\usepackage{makecell}   
\usepackage{ragged2e}   

\usepackage[most]{tcolorbox}
\usepackage{xcolor}
\usepackage[table]{xcolor}
\usepackage[ruled,vlined]{algorithm2e}

\newcommand{\PSS}{\operatorname{PSS}}
\definecolor{takeawaypurple}{HTML}{7CBAD8}
\definecolor{takeawaypurplelight}{HTML}{EBF6FC}
\newtheorem{proposition}{Proposition}

\newtcolorbox{takeawaybox}[1][]{
  enhanced,
  breakable,
  colback=takeawaypurplelight,
  colframe=takeawaypurple,
  boxrule=0.8pt,
  arc=2mm,
  left=1.2mm,
  right=1.2mm,
  top=2.5mm,
  bottom=1.2mm,
  fonttitle=\bfseries,
  coltitle=white,
  colbacktitle=takeawaypurple,
  title=Takeaways,
  attach boxed title to top left={xshift=2mm,yshift=-3mm},
  boxed title style={
    arc=1.5mm,
    boxrule=0pt,
    colframe=takeawaypurple,
    colback=takeawaypurple,
    top=0.2mm,    
    bottom=0.2mm, 
    left=1mm,   
    right=1mm,  
  },
  #1
}

\usepackage{tikz}
\usetikzlibrary{positioning}

\usepackage{caption}
\begin{document}

\ifcolmsubmission
\linenumbers
\fi

\maketitle

\begin{abstract}
Persona prompting has been widely adopted to steer large language models (LLMs) behavior and improve their instruction performance by assigning specific characters. However, identifying an optimal persona is time-consuming, and its impact on output quality remains poorly understood. Prior work has mainly addressed this issue at the prompt level via inference-time strategies, incurring additional computation. In this work, we avoid inference-time prompt search by tackling persona sensitivity during training, aiming to train models that adapt their behavior to diverse personas while preserving task performance. In particular, we find that reinforcement learning with verifiable rewards (RLVR) systematically reduces sensitivity to persona prompts, but also reveals an inherent trade-off of outcome-based optimization: while RLVR improves robustness on tasks with verifiable goals, it can also degrade persona expressivity when needed, e.g., in-character role-playing. To address this limitation, we propose \textbf{PerMix-RLVR}, a persona-mixed RLVR strategy that mitigates the persona robustness–fidelity trade-off, preserving strong robustness to harmful persona variation while enabling faithful persona adoption when required. Concretely, PerMix-RLVR improves persona stability score (PSS) over RLVR by +21.2\% on MATH500, while also enhancing persona fidelity by +11.4\% on PersonaGym.
\end{abstract}
\section{Introduction}
Persona prompting, assigning a role such as \textit{``You are a mathematical expert''}, is widely used to interact with LLMs \citep{pataranutaporn2021ai, bai2022training, liu2023pre, shanahan2023role, kong2023better, luo2024personamath, luz2025helpful, sandwar2025town}. A recurring observation across recent studies is that an appropriate persona can substantially improve task performance, yet the outcome is highly unreliable: the same strategy may also cause severe degradation when the persona is mismatched or poorly specified. For instance, \citet{zheng2024helpful} show that automatically selecting the best persona is difficult, with selection strategies often performing no better than random. Likewise, \citet{kim2024persona} and \citet{luz2025helpful} characterize persona prompting as a \textit{double-edged sword}; high upside, but a non-trivial risk of harming performance.

This volatility makes persona prompting expensive to use in practice. As shown in \autoref{fig:motiv}, (a, b): the task-aligned expert persona often achieves the best performance among the tested variants, and (c): performance further improves as users explore different persona variations. However, the best persona cannot be reliably predicted in advance: one must run full evaluations, wait for the final results, and repeat trial-and-error across multiple prompts to identify a strong choice. To address this, prior work has largely pursued prompt-level interventions. For example, \citet{kim2024persona} propose an ensemble approach that generates both persona-conditioned and neutral responses and then selects between them using an LLM-based evaluator. 

\begin{figure}[t]
\begin{center}
\centering
\small
\centerline{\includegraphics[width=0.9\textwidth]{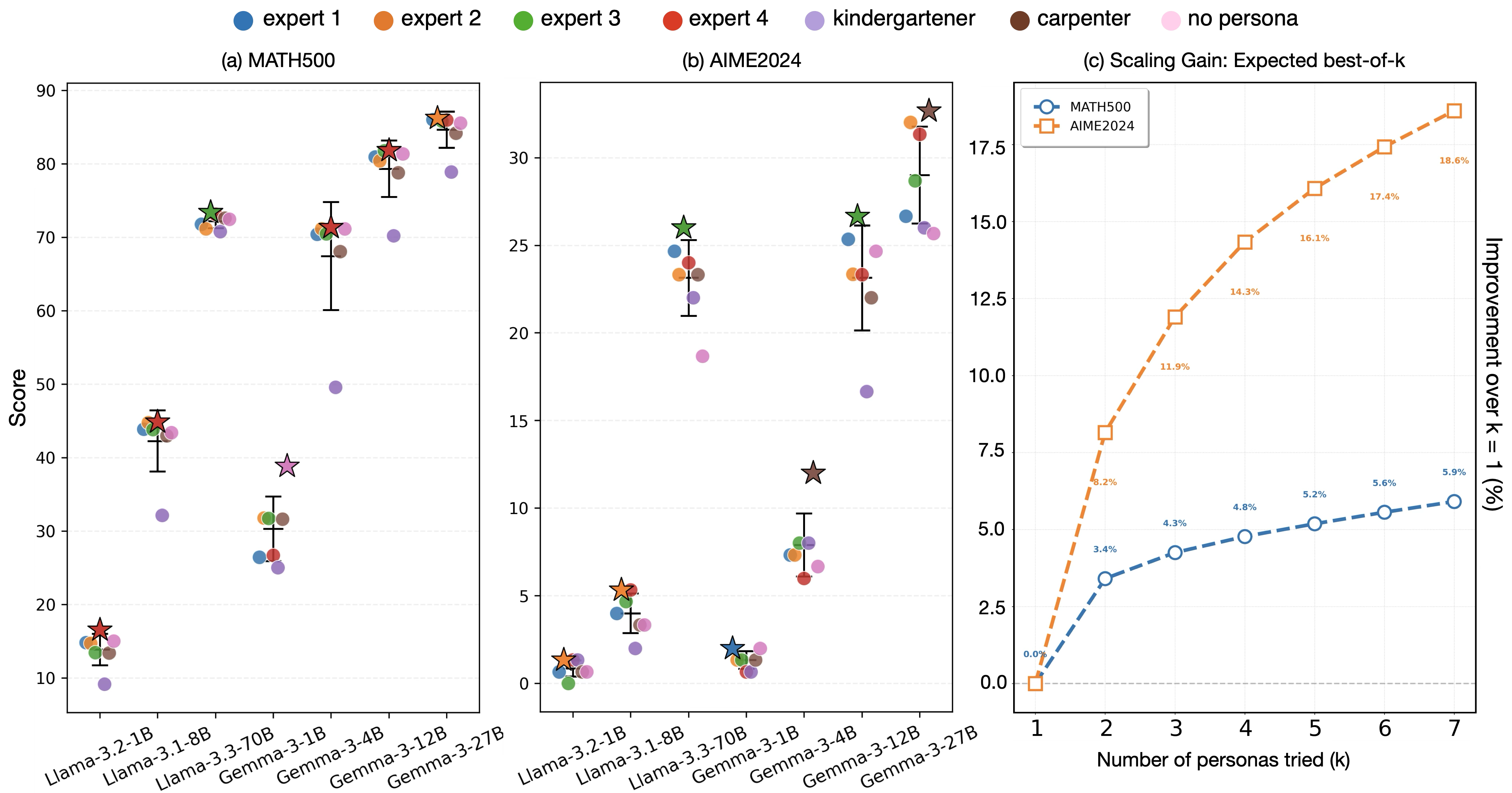}}
\end{center}
\vspace{-15pt}
\caption{\textbf{Persona lottery and the cost of prompt search.}
(a) \textbf{MATH500} and (b) \textbf{AIME2024} show Llama3 and Gemma3 performances under seven different personas (four math \textit{experts}, \textit{kindergartener}, \textit{carpenter}, and \textit{no persona}). Each dot denotes persona; $\star$ and $\circ$ mark the best and worst persona per model, and black bars show mean $\pm$ std across personas, revealing large persona-induced variance. (c) \textbf{Expected best-of-$k$.} We report the expected best score after trying $k$ personas (without replacement), averaged across models.}
\label{fig:motiv}
\end{figure}

We argue for a paradigm shift. Rather than asking how to find the optimal persona for a persona-sensitive model, we ask: \emph{why not find a model that is intrinsically robust to persona variations in the training-time?} From a user perspective, this training-time approach is strictly more attractive: it promises reliable performance without inference-time meticulous prompt engineering or the uncertainty of trial-and-error. Importantly, we contend that persona sensitivity is not an unavoidable or random artifact. Instead, we hypothesize that robustness to persona prompts is closely tied to the model's post-training objective.

Our empirical investigation reveals a clear behavioral divergence: models trained with reinforcement learning with verifiable rewards (RLVR) \citep{lample2019deep, cobbe2021training, le2022coderl, lewkowycz2022solving, lightman2023let} exhibit remarkable robustness to persona variations, whereas alternative post-training recipes we study (e.g., supervised fine-tuning or distillation-based training) show substantially higher sensitivity. We validate this pattern across complex mathematical reasoning benchmarks, model families, and task difficulties.

Besides, our analysis highlights that RLVR alignment emphasizes stability on the tasks with verifiable goals, motivating complementary approaches when persona expressivity is required. This observation frames persona robustness and role-playing fidelity as distinct axes of model behavior, guiding the design of training strategies that balance robustness with controlled persona adoption. 

Our contributions are summarized as follows:
\begin{itemize}[leftmargin=*]
    \vspace{-5pt}
    \item We identify the post-training alignment objective as a key determinant of persona sensitivity in LLMs, showing that models trained with RLVR exhibit systematically reduced sensitivity to persona prompts.
    \vspace{-1pt}
    
    \item We provide theoretical and empirical analysis showing that the RLVR alignment induces persona robustness in goal-oriented tasks while attenuating persona expressivity in role-playing settings, revealing a trade-off between stability and expressivity.
    \vspace{-1pt}
    
    \item We propose \textbf{PerMix-RLVR}, a training strategy that aims to preserve robustness to harmful persona variation while recovering persona expressivity when required. We demonstrate that this strategy improves standard RLVR on math reasoning and role-playing tasks.
\end{itemize}
\section{Related Work}

\subsection{Prompt Engineering for Persona-Aligned Reasoning}

The advent of large-scale pre-trained language models has established prompting as the primary paradigm for interacting with these models \citep{qiu2020pre, mann2020language, gao2020making, lester2021power}. While early work focused on simple template design for few-shot learning, the need for more complex reasoning led to a seminal breakthrough with Chain-of-Thought (CoT) prompting. Parallel to the development of reasoning-focused techniques, \textit{persona prompting} has become a widely disseminated practice in the community. This strategy involves assigning the LLM a specific role or identity, such as an expert, a character, or even an object, with the goal of steering its knowledge, tone, and response style \citep{gu2023systematic, shao2023character, bubeck2023sparks}. The technique is frequently recommended in guides for controlling model behavior and is implicitly used in various applications, from building specialized chatbots to enhancing the safety and alignment of models by instructing them to be helpful and harmless \citep{bai2022training}. The underlying intuition is that a persona helps to contextualize the task, thereby activating a more relevant and effective subspace of the model's vast parametric knowledge. Despite its widespread adoption, persona prompting remains an under-studied phenomenon in a rigorous academic context \citep{battle2024unreasonable, shanahan2023role}. Most empirical studies that employ personas often do so in conjunction with other detailed instructions.

\subsection{Reinforcement Learning with Verifiable Rewards}
Reinforcement learning with verifiable rewards (RLVR) \citep{shao2024deepseekmath, guo2025deepseek, wen2025reinforcement, mroueh2025reinforcement, vojnovic2025alignment} has emerged as a practical paradigm for improving reasoning in domains where correctness can be automatically checked, such as mathematics and code generation. By optimizing policies against an explicit verifier, RLVR can improve task accuracy and calibration without requiring human preference labels. Recent toolkits \citep{vonwerra2020trl, sheng2024hybridflow, feng2025group} and algorithmic variants \citep{shao2024deepseekmath, yu2025dapo, zheng2025group} make RLVR increasingly accessible for post-training open models. Despite this progress, existing RLVR research has primarily focused on improving average task performance and sample efficiency, and has paid comparatively little attention to robustness under systematic prompt perturbations, such as persona changes in the system prompt. Our work bridges this gap by analyzing persona sensitivity through the lens of RLVR.
\section{RLVR Alignment and Persona Robustness \& Fidelity}
\label{sec:prelim}

In this section, we show that RLVR improves persona robustness on verifiable tasks, but can attenuate persona fidelity (i.e., reduce how faithfully the model stays in the character) in settings where persona expressivity is desired.

\subsection{Empirical Evidence: RLVR-Trained Model's Persona Robustness-Fidelity Trade-off}
\label{sec:open-source_exp}

We examine an overall trend by analyzing publicly available post-trained models across different families and scales. Our goal here is not to isolate individual factors, but to observe the macroscopic patterns that emerge from real-world alignment pipelines as they are released. Comparing RLVR and non-RLVR pipelines reveals systematic differences in persona sensitivity and suggests a robustness–fidelity trade-off, where RLVR can improve stability yet reduce adherence to persona prompts. Section~\ref{sec:method} subsequently validates this observation with controlled experiments.

\paragraph{Datasets and models.} We utilize two representative math benchmarks, MATH500 \citep{hendrycks2021measuring} and AIME2024 \citep{aime2024}. For the evaluation of publicly available post-trained models, we examine three primary model families: Qwen3 (0.6B, 1.7B, 4B, 8B, 32B) \citep{yang2025qwen3}, Llama3.1\,\&\,3.2 (1B, 8B, 70B) \citep{grattafiori2024llama}, and Gemma3 (1B, 4B, 12B, 27B) \citep{team2025gemma}. We note that Qwen3-32B has been post-trained with RLVR, and the other Qwen3 models have been distilled from Qwen3-32B/235B-A22B without additional post-training, as highlighted in \citet{yang2025qwen3}. Additionally, the Gemma3 models undergo post-training with reinforcement learning using mixed reward signals, whereas the Llama models do not include RLVR in their post-training pipeline.

\paragraph{Prompt design.}
For the evaluation of persona robustness, we use a fixed set of system personas, $|\mathcal{P}_{\text{eval}}|=16$, spanning four categories (STEM experts, education levels, character traits, and job roles), and measure performance under each persona setting. We provide the full persona pool and prompt templates (including the disjoint training persona set $\mathcal{P}_{\text{train}}$ used in Section~\ref{sec:method}) in Appendix~\ref{app:sec:prompt_construction}.

\paragraph{Metric.} We introduce Persona Stability Score (PSS) which measures how the performance varies according to the persona prompt. We define PSS as follows:
\begin{equation}
  \PSS(m,d)
  \;=\;
  \frac{\displaystyle \min_{p \in \mathcal{P}} \; a_{m,p,d}}
       {\displaystyle \max_{p \in \mathcal{P}} \; a_{m,p,d}}
  \;\in\; [0,1]
\end{equation}
\noindent 
where $a_{m,p,d}\in[0,1]$ is the performance (e.g., accuracy) of a model $m$ with a persona $p$ on a dataset $d$.
By construction, larger values (i.e., closer to $1$) indicate \emph{persona-stable} behavior—performance is similar across personas—whereas smaller values (i.e., closer to $0$) indicate \emph{persona-sensitive} behavior with large disparities between the best and worst personas. Given a dataset of $N$ instances $\{x_i\}_{i=1}^N$, we generate a single output $y_i \sim \pi(\cdot \mid x_i)$ for each instance. We report the single-sample accuracy, equivalent to $\mathrm{Pass@}1$.

\begin{figure*}[t!]
\vspace{-10pt}
\centering
\small
\centerline{\includegraphics[width=0.75\textwidth]{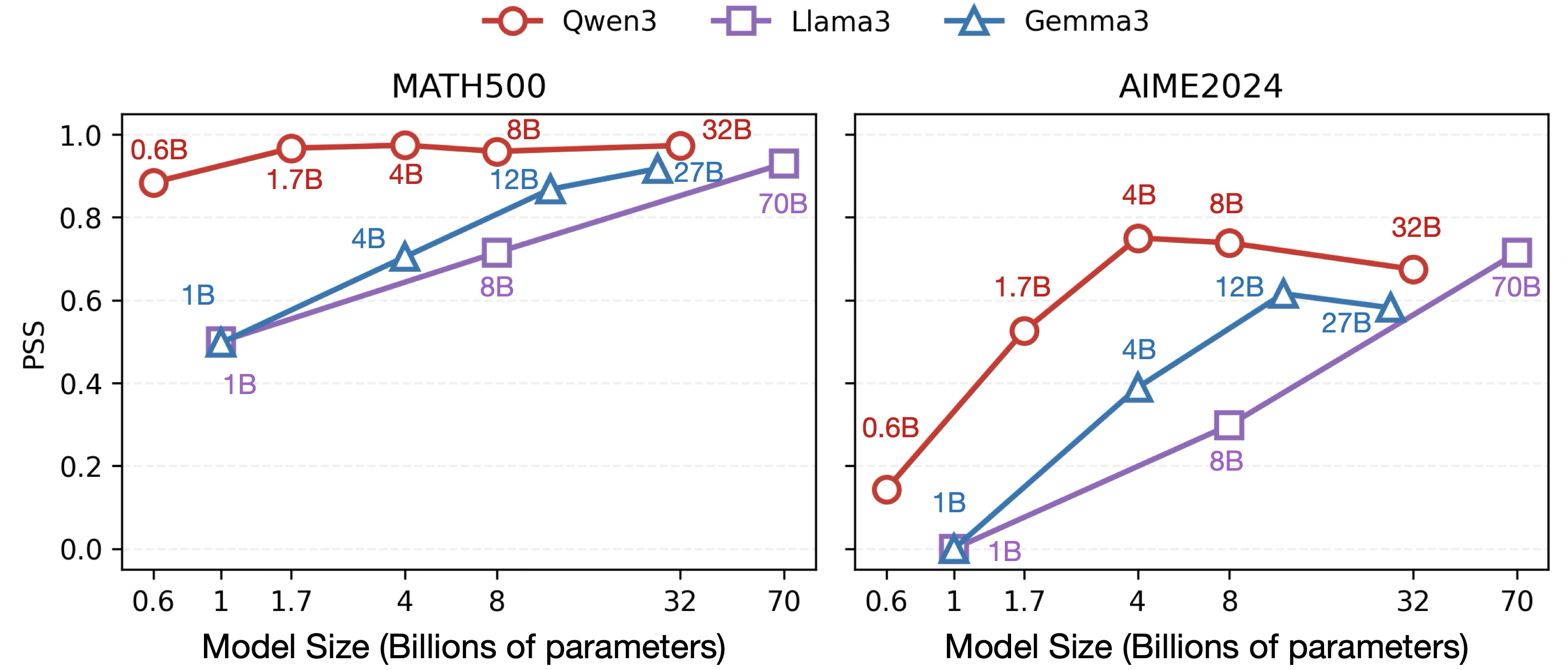}}
\caption{\textbf{PSS comparison across model families and scales.} RLVR-trained families consistently exhibit higher PSS, indicating greater robustness to persona-induced variations. Full results are provided in Appendix~\ref{app:sec:ablation_study}.}
\label{fig:pss_graph}
\vspace{-5pt}
\end{figure*}

\paragraph{RLVR-trained model shows the highest PSS.}
The post-training recipe appears to correlate with persona robustness as shown in Figure~\ref{fig:pss_graph}. Qwen3 models consistently achieve higher PSS than Llama3 and Gemma3 across both benchmarks. Given public technical reports, this pattern is consistent with the hypothesis that RLVR may improve robustness to persona variation. Gemma3 shows a modest stability advantage over Llama3, which may reflect its hybrid post-training recipe combining reward-model-based optimization with correctness-verifiable signals, whereas Llama3 primarily relies on non-RLVR post-training.

\begin{wraptable}{r}{0.52\textwidth}
\centering
\small
\vspace{-10pt}
\addtolength{\tabcolsep}{-2pt}
\resizebox{0.52\textwidth}{!}{
\vspace{-5pt}
\begin{tabular}{lccc}
\toprule
\textbf{Model} & \textbf{WR} & \textbf{WR\,(no-tie)} & \textbf{Net margin} \\
\midrule
Llama3.1 (non-RLVR) & $50.2_{\,\pm\,9.0}$ & $60.3$ & $+17.2$ \\
Qwen3 (RLVR)           & $33.0_{\,\pm\,6.7}$ & $39.7$ & $-17.2$ \\
Tie                   & $16.8_{\,\pm\,4.3}$ & --     & --      \\
\bottomrule
\end{tabular}
}
\vspace{-5pt}
\caption{\small{\textbf{Pairwise LLM-judge comparison under the \textit{kindergartener} persona on MATH500.}
WR denotes the win rate\,(\%), and WR (no-tie) is computed by excluding ties: $\text{win}/(\text{win}+\text{loss})$ using mean counts. 
Net margin is $\text{win}-\text{loss}$ (mean).}}
\label{tab:kindergarten_pairwise_wrap}
\vspace{-10pt}
\end{wraptable}

\paragraph{RLVR attenuates persona fidelity.}
To disentangle persona fidelity from task competence, we run an additional analysis on MATH500 to compare which post-trained model better preserves the intended persona while solving the same problems. Although RLVR improves robustness on verifiable tasks, it can weaken persona fidelity where in-character behavior is essential. In Table~\ref{tab:kindergarten_pairwise_wrap}, we used an LLM judge, DeepSeek-v3.2 \citep{liu2025deepseek}, to assess which model produced more \textit{kindergartener}-like outputs given the \textit{kindergartener} persona prompt. It strictly preferred the outputs of Llama3.1-8B-Instruct over the RLVR-trained Qwen3-32B, which is much larger than Llama, in 50.2\% of pairwise comparisons. 
This suggests the outcome-driven optimization can favor standardized problem-solving over persona-specific cues, motivating complementary training when faithful persona adoption matters.

\begin{takeawaybox}[title=Takeaway 1: Persona Robustness–Fidelity Trade-off]
\begin{enumerate}[label=(\alph*), leftmargin=1.5em, itemsep=0.2em, topsep=0.2em]
    \item \textbf{Post-training objectives can pull persona behavior in different directions: }RLVR pipelines correlate with higher persona robustness (PSS) on verifiable tasks.
    \item \textbf{RLVR may attenuate persona fidelity in role-playing: }an RLVR-trained model is judged less “kindergartener-like” than a non-RLVR baseline.
\end{enumerate}
\end{takeawaybox}

To better understand the mechanism behind Takeaway 1, we mathematically formalize RLVR and its interaction with persona prompts.
 
\subsection{Theoretical Background: Reinforcement Learning with Verifiable Rewards}
\label{sec:rlvr_background}
 
Let $x$ denote a problem and $p$ a persona prompt.
The model generates an output $y$, from which we extract a reasoning trajectory $z$ and final answer $a(z)$.
The verifier $V(z) = V(a(z)) \in \{0, 1\}$ depends only on the answer; we write $\mathcal{C}(x) = \{z : V(z) = 1\}$ for the correct set.
Standard RLVR optimizes over $x$ without persona:
\begin{equation}
\label{eq:rlvr_objective}
    J(\pi) = \mathbb{E}_{z \sim \pi(\cdot|x)}\!\big[V(z)\big] - \beta\, D_{\mathrm{KL}}\!\big(\pi(\cdot|x)\,\|\,\pi_{\mathrm{ref}}(\cdot|x)\big).
\end{equation}

The solution $\pi^*$ of Eq.~\eqref{eq:rlvr_objective} reweights the reference 
density by $e^{1/\beta}$ on correct trajectories (Appendix~\ref{app:theory}). 
This reweighting depends only on the answer, not on the reasoning path or its 
style. However, because different styles lead to different correctness rates, 
the answer-based reweighting implicitly acts as a filter over styles, a 
mechanism we formalize next.
 
\subsection{Theoretical Motivation: RLVR as a Style Competence Filter}
\label{sec:theory}

We present a stylized model for persona robustness under RL on verifiable tasks. Our analysis shows that the answer-based RL acts as a reweighting mechanism of a style prior by the competence factor. We use this framework to interpret the empirical behavior of persona-conditioned training, as shown in Figure~\ref{fig:concept}.

\paragraph{Style model.}
For analysis, we introduce a stylized latent-variable model of the output and write $y\!=\!(s, z)$, where $s$ denotes an unobserved style variable and $z$ denotes the reasoning trajectory that determines the final answer. We assume the factorization
\[
\pi(y \mid c) = \pi(s \mid c)\,\pi(z \mid s, x),
\]
where $c$ denotes the conditioning context ($c\!=\!x$ for standard RLVR, and $c\!= \!(x,p)$ for persona-conditioned RLVR). Here, the persona prompt $p$ shifts the style prior $\pi(s \mid c)$, while $\pi(z \mid s,x)$ depends on the style-problem pair, not directly on the persona identity. This provides a tractable abstraction for analyzing persona robustness on verifiable tasks.
 
\paragraph{Filtering mechanism.}
Under this factorization, RLVR performs a Bayesian update on the style distribution (detailed in Appendix~\ref{app:style_posterior}):
\begin{equation}
\label{eq:style_posterior}
    \pi^*(s \mid c) \;\propto\; \underbrace{\pi_{\mathrm{ref}}(s \mid c)}_{\text{style prior}} \;\cdot\; \underbrace{\big[1 + (e^{1/\beta} - 1)\,\mu(s, x)\big]}_{\text{style competence}},
\end{equation}
where $\mu(s, x) = \pi_{\mathrm{ref}}(\mathcal{C}(x) \mid s, x)$ is the probability that the reference model reasons correctly under style $s$ on problem $x$. The competence term increases monotonically with $\mu(s,x)$: styles that support correct reasoning are upweighted, while styles with $\mu(s,x) \approx 0$ receive little mass and are filtered out as $\beta \to 0$. Under this stylized model, the resulting filter depends on the style--problem pair rather than directly on the persona identity.

\begin{figure*}[t!]
\vspace{-10pt}
\centering
\small
\centerline{\includegraphics[width=0.90\textwidth]{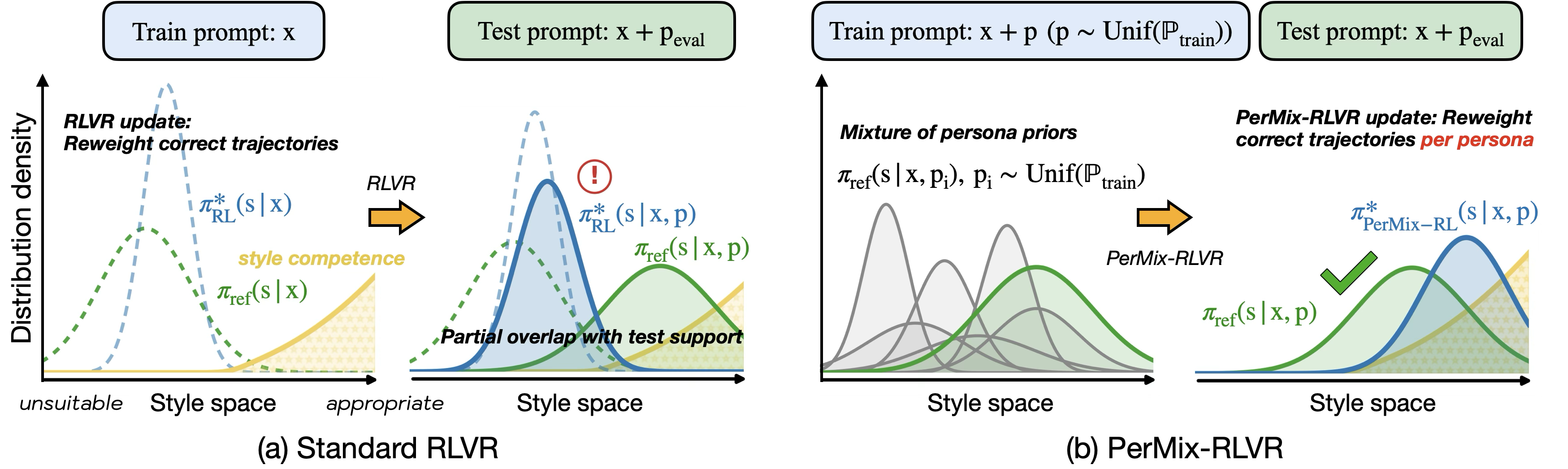}}
\caption{\textbf{Standard RLVR vs. PerMix-RLVR.} (a) RLVR reweights $\pi_{\mathrm{ref}}(s|x)$ toward competent styles, yielding $\pi_{\mathrm{RL}}(s|x)$. A test-time persona shifts the prior to $\pi_{\mathrm{ref}}(s|x,p)$, causing support mismatch for standard RLVR. (b) PerMix-RLVR is trained on diverse persona-conditioned priors to produce $\pi_{\text{PerMix-RL}}(s|x,p)$ aligned with persona-shifted support. This indicates better coverage under the persona-shifted prior $\pi_{\mathrm{ref}}(s|x,p)$. The yellow hatched region represents style competence.}
\label{fig:concept}
\end{figure*}

\paragraph{Standard RLVR vs.~Persona-conditioned RLVR.}
Eq.~\ref{eq:style_posterior} upweights styles with high $\mu(s,x)$ and suppresses those with $\mu(s,x) \approx 0$. Standard RLVR ($c = x$) learns this reweighting under the reference style prior $\pi_{\mathrm{ref}}(s \mid x)$, which concentrates on the model's default styles. At test time, a persona prompt shifts the style prior, which may place mass on styles 
outside the trained styles where the learned competence 
filter has not been calibrated (Figure~\ref{fig:concept}(a)).

Persona-conditioned training ($c = (x,p)$) broadens training exposure to persona-shifted style priors, providing a potential for its stronger robustness and better persona-consistent behavior in our experiments (Section~\ref{sec:method}).

This advantage can be made precise: under the KL-regularized optimal 
policy with $V(z) \in \{0,1\}$, persona-conditioned RLVR admits a 
closed-form expression for persona-conditional accuracy that 
formalizes the robustness effect.
 
\begin{proposition}[Accuracy and Robustness]
\label{prop:acc_robust}
Let $\mu_p = \pi_{\mathrm{ref}}(\mathcal{C}(x)|x,p)$ and $\alpha(\mu, \beta) = e^{1/\beta}\mu \,/\, (e^{1/\beta}\mu + 1 - \mu)$.
Under persona-conditioned RLVR:
\begin{enumerate}
    \item[\textnormal{(i)}] \textbf{Accuracy:} $\mathrm{Acc}(\pi^*, x, p) = \alpha(\mu_p, \beta)$, depending on $p$ only through $\mu_p$.
    \item[\textnormal{(ii)}] \textbf{Improvement:} $\alpha(\mu, \beta) > \mu$ for all $\beta > 0$.
    \item[\textnormal{(iii)}] \textbf{Robustness:} $\PSS(\pi^*) \geq \PSS(\pi_{\mathrm{ref}})$, with $\PSS \to 1$ as $\beta \to 0$ provided $\mu_{\min} > 0$.
\end{enumerate}
\end{proposition}
 
\noindent\textit{Proof.} See Appendix~\ref{app:proof_proposition}. 

The PSS bound in (iii) follows from the concavity of $\alpha$: personas with lower $\mu_p$ benefit more, which compress the performance gap.

\begin{takeawaybox}[title=Takeaway 2: Why Persona-Conditioned RLVR Helps]
\begin{enumerate}[label=(\alph*), leftmargin=1.5em, itemsep=0.2em, topsep=0.2em]
    \item \textbf{RLVR induces a style-selective update:} it favors styles that are more compatible with correct reasoning.
    \item \textbf{This improves persona robustness on verifiable tasks:} compatible persona-induced styles survive, while incompatible ones are suppressed.
    \item \textbf{Persona-conditioned RLVR extends robustness to persona-shifted styles:} by learning the same update under diverse persona-shifted priors, it applies over a broader support.
\end{enumerate}
\end{takeawaybox}
\section{PerMix-RLVR: Persona-Conditioned Verifiable-Reward Alignment}
\label{sec:method}

\subsection{Algorithm Design}

Based on the theoretical foundations laid out in Section~\ref{sec:theory}, we have demonstrated that persona-conditioned training holds significant potential for enhancing persona robustness while mitigating the degradation of persona expressivity typically observed in standard RLVR. Thus, we introduce \textbf{PerMix-RLVR}, a practical framework designed to operationalize persona-conditioned training.

To enable the model to effectively learn a calibrated latent style filter through RLVR, we curate a training persona pool $\mathcal{P}_{\text{train}}$, encompassing a wide range of persona categories. During the training process, for each instance, a persona is extracted via uniform sampling and integrated into the input as a system prompt. Aside from this specialized input formatting, the underlying training methodology remains consistent with standard RLVR. Specifically, we employ the Group Relative Policy Optimization (GRPO)~\citep{shao2024deepseekmath} approach to perform the verifiable-reward alignment. The whole training process is outlined in Algorithm~\ref{alg:permix-rlvr}.

\subsection{Empirical Evaluation}
\vspace{-5pt}
To empirically evaluate the proposed \textbf{PerMix-RLVR}, we design a series of controlled experiments. As established in our earlier analysis, persona sensitivity is a systematic consequence closely linked to a model's post-training pipeline. Therefore, to isolate the causal effect of different post-training strategies, we ensure that all variants originated from the same base checkpoint and are trained under matched training budgets. Through these experiments, we evaluate \textbf{PerMix-RLVR} by comparing it against common post-training paradigms: supervised fine-tuning (SFT), knowledge distillation (KD), and reinforcement learning with verifiable rewards (RLVR).

\begin{table*}[t!]
\centering
\small
\addtolength{\tabcolsep}{-1.5pt}
\resizebox{\linewidth}{!}{
\begin{tabular}{l|cc|cccc|cccc|c}
\toprule
\multirow{2}{*}{\textbf{Post-training}}& \multirow{2}{*}{\textbf{Input}} & \multirow{2}{*}{\textbf{Objective ($\mathcal{J}$)}} & \multicolumn{4}{c|}{\textbf{GSM8K (i.d.)}} & \multicolumn{4}{c|}{\textbf{MATH500 (o.o.d.)}} & \textbf{PersonaGym} \\
 &  &  & \texttt{Worst} & \texttt{Best} & \texttt{Acc.\,Mean} & \texttt{PSS} & \texttt{Worst} & \texttt{Best} & \texttt{Acc.\,Mean} & \texttt{PSS} & \texttt{Consistency} \\
\midrule

Base ($M_0$) & $x$ & -
& 81.1 & 84.9 & 83.5$_{\,\pm\,1.1}$  & 0.956 & 28.6 & 46.0 & 42.5$_{\,\pm\,3.8}$ & 0.619 & 3.19 \\

\midrule

SFT & $x$ & $\mathcal{L}_{CE}(y)$
& 73.4 & 77.2 & 75.8$_{\,\pm\,1.1}$ & 0.951 & 30.3 & 37.2 & 34.6$_{\,\pm\,1.7}$ & 0.805 & 2.53 \\

\midrule

\multirow{2}{*}{PerMix-SFT} & \multirow{2}{*}{$x$+$p$} & $\mathcal{L}_{CE}(y)$
& 74.3 & 76.2 & 75.2$_{\,\pm\,0.5}$ & \underline{0.974} & 33.0 & 35.3 & 34.2$_{\,\pm\,0.6}$ & \textbf{0.880} & 2.88 \\

 &  & $\mathcal{L}_{CE}(y_s)$
& 71.0 & 81.5 & 78.7$_{\,\pm\,2.8}$ & 0.871 & 24.8 & 41.2 & 37.4$_{\,\pm\,3.9}$ & 0.591 & \underline{3.31} \\

\midrule

RLVR & $x$ & $\mathbb{E}[V - \beta D_{KL}]$
& \underline{84.0} & \textbf{87.6} & \underline{86.1}$_{\,\pm\,0.9}$ & 0.959 & \underline{34.0} & \textbf{49.6} & \underline{46.8}$_{\,\pm\,3.8}$ & 0.675 & 3.06 \\

\rowcolor{blue!15}
\textbf{PerMix-RLVR} (ours)& $x$+$p$ & $\mathbb{E}[V - \beta D_{KL}]$
& \textbf{84.9} & \underline{87.0} & \textbf{86.4}$_{\,\pm\,0.6}$ & \textbf{0.975} & \textbf{41.0} & \underline{48.6} & \textbf{47.1}$_{\,\pm\,1.7}$ & \underline{0.818} & \textbf{3.41} \\

\bottomrule
\end{tabular}
}
\caption{\textbf{Comparison of teacher-free post-training strategies on GSM8K and MATH500 (accuracy\,\%, PSS) and PersonaGym (role-playing consistency).} $x$ is an input instance, while $p$ is an input persona prompt. $\mathcal{L}_{CE}$ is a cross-entropy loss on a ground-truth $y$ (or persona-styled $y_s$ rewritten by GPT-5-mini). RLVR uses verifier reward $V$ along with KL regularization $\beta D_{KL}$ for optimization. \textbf{Bold} and \underline{underline} indicate the best and second best performance, respectively.}
\label{tab:controlled_baselines_in_domain_teacher_free}
\end{table*}


\subsubsection{Experiment Setup}
\vspace{-3pt}
\paragraph{Models and benchmarks.}
All variants start from Llama3.1-8B-Instruct ($M_0$) and are trained on the same GSM8K \citep{cobbe2021training} training set.
We evaluate on four benchmarks: GSM8K, MATH500, LiveCodeBench \citep{jain2024livecodebench}, and PersonaGym \citep{samuel-etal-2025-personagym}.
GSM8K serves as the in-distribution evaluation, while MATH500 and LiveCodeBench v6 are held-out reasoning benchmarks used to measure generalization and persona robustness under different system personas.
In contrast, PersonaGym assesses persona-conditioned expressivity. Among its five evaluation tasks, we identify the \textbf{Persona Consistency} task as the primary metric that best represents the model's ability to maintain its assigned identity. Consequently, we focus on evaluating scores for this specific task to measure the role-playing performance across different post-training methods. For completeness, we also evaluate the models on all other PersonaGym tasks, where the detailed results for these metrics are provided in Appendix~\ref{app:subsec:persona_gym_detail}.

\paragraph{Post-training strategies.}
To isolate the causal effects of different alignment objectives, we compare several representative post-training paradigms starting from the same, instruction-tuned base checkpoint $M_0$. For SFT, the base model is trained on the original GSM8K rationales and answers $y$. We further design \textbf{PerMix-SFT} as a persona-conditioned baseline, which utilizes randomly sampled persona-conditioned inputs $x$+$p$. We conduct with two variants: one is supervised by the original GSM8K rationales $y$, whereas the other is supervised by $y_s$, which consists of target responses rewritten by GPT-5-mini to be persona-consistent. For distillation, we employ sequence-level knowledge distillation (SeqKD) \citep{kim2016sequence} rather than token-level distillation to account for potential token differences between the teacher and student models. In this setup, the student is fine-tuned on teacher-generated sequences. Additionally, we design a persona-conditioned baseline for the distillation strategy, where the teacher model generates sequences while being explicitly conditioned on a sampled persona prompt. We utilize three large-scale teachers (Qwen3-32B, Llama3.1-70B, and Gemma3-27B) to examine whether persona robustness transfers across different model families. The training personas are strictly disjoint from the evaluation personas to prevent leakage. Each experiment was evaluated by averaging the results over five runs per persona, and the task performance was measured using Pass@1 accuracy. Appendix~\ref{app:sec:implementation} describes more implementation details.

\subsubsection{Teacher-Free Post-Training Evaluation}
Table~\ref{tab:controlled_baselines_in_domain_teacher_free} presents the accuracy and persona robustness in the in-domain (math) setting, which is identical to the training dataset, along with the role-playing capability of teacher-free methods. Compared to the base model $M_0$, while standard RLVR shows improvements in both accuracy and persona robustness, \textbf{PerMix-RLVR} significantly outperforms these gains. Notably, \textbf{PerMix-RLVR} substantially elevates the worst-case performance compared to RLVR, which leads to a marked increase in PSS across both the in-distribution GSM8K and out-of-distribution MATH500 benchmarks. 

Intriguingly, our evaluation of PerMix-SFT reveals a trade-off: while using a persona-agnostic target $y$ for an input ($x$, $p$) inherently minimizes persona influence and increases PSS, it is accompanied by a significant degradation in overall accuracy. Moreover, shifting the target to stylized $y_s$ results in a drastic decline in PSS, indicating high sensitivity to persona variations.
In contrast, \textbf{PerMix-RLVR} maintains superior accuracy alongside highest PSS, providing empirical evidence for our theoretical hypothesis that persona-conditioned RLVR more effectively calibrates the model's style competence filter. Furthermore, evaluation on PersonaGym reveals that \textbf{PerMix-RLVR} achieves a higher consistency score even when compared to PerMix-SFT with $y_s$, a baseline specifically optimized for persona fidelity.

\begin{table*}[t!]
\centering
\small
\addtolength{\tabcolsep}{-2pt}    
\resizebox{\linewidth}{!}{
\begin{tabular}{lc|cc|cccc|cccc|c}
\toprule
\multirow{2}{*}{\textbf{Post-training}} & \multirow{2}{*}{\textbf{Teacher}} & \multirow{2}{*}{\textbf{Input}} & \multirow{2}{*}{\textbf{Objective ($\mathcal{J}$)}} & \multicolumn{4}{c|}{\textbf{GSM8K (i.d.)}} & \multicolumn{4}{c|}{\textbf{MATH500 (o.o.d.)}} & \textbf{PersonaGym} \\
 &  &  &  & \texttt{Worst} & \texttt{Best} & \texttt{Acc.\,Mean} & \texttt{PSS} & \texttt{Worst} & \texttt{Best} & \texttt{Acc.\,Mean} & \texttt{PSS} & \texttt{Consistency} \\
 
\midrule

\multirow{2}{*}{SeqKD} & \multirow{2}{*}{Qwen3-32B} & $x$ & $\mathcal{L}_{CE}(y^t)$
& 80.4 & 84.7 & 83.2$_{\,\pm\,0.9}$ & 0.949 & 37.6 & 48.2 & 45.8$_{\,\pm\,2.6}$ & 0.764 & \underline{3.12} \\

& & $x$+$p$ & $\mathcal{L}_{CE}(y^t_p)$
& 81.0 & 85.3 & 84.3$_{\,\pm\,1.2}$ & {0.949} & \underline{40.7} & 47.8 & 45.5$_{\,\pm\,1.6}$ & \textbf{0.839} & 3.06 \\

\midrule

\multirow{2}{*}{SeqKD} & \multirow{2}{*}{Llama3.1-70B} & $x$ & $\mathcal{L}_{CE}(y^t)$
& 83.7 & \underline{87.5} & \textbf{86.6}$_{\,\pm\,0.9}$ & 0.957 & 35.6 & 48.0 & 45.7$_{\,\pm\,2.9}$ & 0.720 & 2.72 \\

& & $x$+$p$ & $\mathcal{L}_{CE}(y^t_p)$
& 82.9 & 87.3 & \underline{86.5}$_{\,\pm\,1.0}$ & {0.950} & 35.2 & 47.7 & 45.2$_{\,\pm\,2.9}$ & 0.728 & 3.09 \\

\midrule

\multirow{2}{*}{SeqKD} & \multirow{2}{*}{Gemma3-27B} & $x$ & $\mathcal{L}_{CE}(y^t)$
& 62.1 & 85.2 & 82.0$_{\,\pm\,5.6}$ & 0.729 & 31.7 & 47.2 & 45.3$_{\,\pm\,3.8}$ & 0.656 & 3.06 \\

& & $x$+$p$ & $\mathcal{L}_{CE}(y^t_p)$
& 81.5 & 85.8 & 84.9$_{\,\pm\,1.2}$ & {0.950} & 21.4 & 46.3 & 42.3$_{\,\pm\,6.2}$ & 0.452 & 2.97 \\

\midrule

RLVR & - & $x$ & $\mathbb{E}[V - \beta D_{KL}]$
& \underline{84.0} & \textbf{87.6} & {86.1}$_{\,\pm\,0.9}$ & \underline{0.959} & 34.0 & \textbf{49.6} & \underline{46.8}$_{\,\pm\,3.8}$ & 0.675 & 3.06 \\

\rowcolor{blue!15}
\textbf{PerMix-RLVR} (ours)\!\!\! & - & $x$+$p$ & $\mathbb{E}[V - \beta D_{KL}]$
& \textbf{84.9} & 87.0 & {86.4}$_{\,\pm\,0.6}$ & \textbf{0.975} & \textbf{41.0} & \underline{48.6} & \textbf{47.1}$_{\,\pm\,1.7}$ & \underline{0.818} & \textbf{3.41} \\

\bottomrule
\end{tabular}
}
\caption{\textbf{Comparison of teacher-based distillation and RLVR on GSM8K and MATH500 (accuracy\,\%, PSS) and PersonaGym (role-playing consistency).} Llama3.1-8B-Instruct is the student, and $y^t$ denotes teacher outputs (or $y^t_p$ when conditioned on a system persona).}
\label{tab:controlled_baselines_in_domain_teacher_based}
\end{table*}


\subsubsection{Teacher-Based Distillation Evaluation}
In Table~\ref{tab:controlled_baselines_in_domain_teacher_based}, we report the comparison between \textbf{PerMix-RLVR} and several teacher-based distillation baselines that depend on various external teacher models. Despite being a teacher-free approach, \textbf{PerMix-RLVR} achieves accuracy and PSS levels comparable to distillation methods that rely on much larger teachers. Notably, \textbf{PerMix-RLVR} continues to demonstrate its strength in the lower tail of performance, achieving the highest worst-case accuracies. Another significant finding is that \textbf{PerMix-RLVR} records a higher PersonaGym consistency score than the distillation baselines trained on $y^t_p$.

In conclusion, these results further validate that \textbf{PerMix-RLVR} is an effective post-training strategy for enhancing persona fidelity while maintaining the most balanced trade-off between task accuracy and persona robustness across all studied paradigms.


\begin{table*}[t]
\centering
\small
\resizebox{0.87\linewidth}{!}{
\begin{tabular}{l|cc|cccc|cccc}
\toprule
\multirow{2.5}{*}{\textbf{Post-training}}
& \multirow{2.5}{*}{\textbf{Input}}
& \multirow{2.5}{*}{\textbf{Objective ($\mathcal{J}$)}}
& \multicolumn{4}{c|}{\textbf{Easy}}
& \multicolumn{4}{c}{\textbf{Medium}} \\
\cmidrule(lr){4-7} \cmidrule(lr){8-11}
&  & 
& \texttt{Worst} & \texttt{Best} & \texttt{Acc.\,Mean} & \texttt{PSS}
& \texttt{Worst} & \texttt{Best} & \texttt{Acc.\,Mean} & \texttt{PSS} \\
\midrule

Base ($M_0$)
& $x$ & -
& 22.5 & 30.4 & 25.4$_{\,\pm\,1.9}$ & 0.741
& 0.6 & \underline{3.6} & 2.1$_{\,\pm\,0.9}$ & 0.167 \\

\midrule
\multicolumn{11}{l}{\textbf{\textit{\textcolor{gray}{Supervised Fine-tuning}}}} \\
\midrule

SFT
& $x$ & $\mathcal{L}_{CE}(y)$
& 19.0 & 25.4 & 22.5$_{\,\pm\,1.6}$ & 0.750
& 0.6 & 3.0 & 1.7$_{\,\pm\,0.7}$ & 0.200 \\

\midrule

\multirow{2}{*}{PerMix-SFT}
& \multirow{2}{*}{$x$+$p$} & $\mathcal{L}_{CE}(y)$
& 21.1 & 28.2 & 25.0$_{\,\pm\,1.9}$ & 0.750
& 0.6 & 3.0 & 1.7$_{\,\pm\,0.6}$ & 0.200 \\

& & $\mathcal{L}_{CE}(y_s)$
& 18.3 & 24.7 & 21.2$_{\,\pm\,1.9}$ & 0.743
& 0.8 & 2.9 & 1.8$_{\,\pm\,0.5}$ & 0.292 \\

\midrule
\multicolumn{11}{l}{\textbf{\textit{\textcolor{gray}{Distillation}}}} \\
\midrule

\multirow{2}{*}{SeqKD-Qwen3-32B}
& $x$ & $\mathcal{L}_{CE}(y^t)$
& 23.9 & 29.6 & 25.9$_{\,\pm\,1.6}$ & \textbf{0.810}
& 0.6 & \underline{3.6} & 2.0$_{\,\pm\,0.9}$ & 0.167 \\

& $x$+$p$ & $\mathcal{L}_{CE}(y^t_p)$
& 21.8 & 29.6 & 25.1$_{\,\pm\,2.1}$ & 0.738
& 0.0 & 2.9 & 2.0$_{\,\pm\,0.7}$ & 0.000 \\

\midrule

\multirow{2}{*}{SeqKD-Llama3.1-70B}
& $x$ & $\mathcal{L}_{CE}(y^t)$
& \underline{24.7} & \underline{31.7} & \underline{27.9}$_{\,\pm\,1.8}$ & 0.778
& 1.2 & 3.0 & \textbf{2.3}$_{\,\pm\,0.5}$ & \underline{0.400} \\

& $x$+$p$ & $\mathcal{L}_{CE}(y^t_p)$
& \textbf{25.5} & \textbf{32.4} & \textbf{28.4}$_{\,\pm\,2.1}$ & 0.787
& \underline{1.3} & 3.0 & \underline{2.2}$_{\,\pm\,0.5}$ & \textbf{0.440} \\
\midrule
\multirow{2}{*}{SeqKD-Gemma3-27B}
& $x$ & $\mathcal{L}_{CE}(y^t)$
& 23.9 & 29.6 & 26.6$_{\,\pm\,2.1}$ & \textbf{0.810}
& \textbf{1.7} & 3.3 & 2.0$_{\,\pm\,0.6}$ & 0.321 \\

& $x$+$p$ & $\mathcal{L}_{CE}(y^t_p)$
& 23.2 & 29.3 & 25.6$_{\,\pm\,1.8}$ & \underline{0.793}
& 0.6 & 3.0 & 1.8$_{\,\pm\,0.7}$ & 0.200 \\

\midrule
\multicolumn{11}{l}{\textbf{\textit{\textcolor{gray}{RL with Verifiable Rewards}}}} \\
\midrule

RLVR
& $x$ & $\mathbb{E}[V - \beta D_{KL}]$
& 21.8 & 30.3 & 26.4$_{\,\pm\,2.2}$ & 0.721
& 1.2 & 3.0 & 2.1$_{\,\pm\,0.7}$ & \underline{0.400} \\

\rowcolor{blue!15}
\textbf{PerMix-RLVR} (ours)
& $x$+$p$ & $\mathbb{E}[V - \beta D_{KL}]$
& 23.2 & 30.3 & 26.4$_{\,\pm\,1.8}$ & 0.767
& \underline{1.3} & \textbf{3.7} & \textbf{2.3}$_{\,\pm\,0.8}$ & 0.355 \\

\bottomrule
\end{tabular}
}
\caption{\textbf{Controlled post-training comparison on LiveCodeBench.} Standard RLVR can \emph{reduce} PSS under domain shift, whereas \textbf{PerMix-RLVR} consistently improves both mean accuracy and PSS across splits, matching or outperforming teacher-free baselines and remaining competitive with teacher-based distillation.}
\label{tab:controlled_baselines_code_easy_medium}
\end{table*}

\subsubsection{Cross-Domain Evaluation}
Finally, we evaluate cross-domain transfer to code generation and reasoning using LiveCodeBench, which includes difficulty strata (Easy/Medium). We use the same evaluation protocol as previous experiments. Table~\ref{tab:controlled_baselines_code_easy_medium} reveals a domain-shift effect in code where standard RLVR appears particularly vulnerable. This supports the finding that verifier-driven optimization alone does not necessarily preserve persona stability under a distributionally different domain. In contrast, \textbf{PerMix-RLVR} consistently improves persona robustness over the base model $M_0$ on both Easy and Medium splits. While distillation methods show strengths in both mean accuracy and persona robustness, \textbf{PerMix-RLVR} remains the most competitive teacher-free approach by significantly enhancing worst-case performance. Considering the high computational budget required for large-scale teacher models in distillation, the fact that \textbf{PerMix-RLVR} achieves comparable performance and robustness proves its effectiveness even under domain shift. We also evaluate our method on more challenging benchmarks, namely AIME2024 for in-domain tasks and the LiveCodeBench Hard split for out-of-domain tasks, whose results are provided in Appendix~\ref{app:subsec:study_hard_tasks}.

\begin{takeawaybox}[title=Takeaway 3: PerMix-RLVR Offers the Most Balanced Trade-off]
\begin{enumerate}[label=(\alph*), leftmargin=1.5em, itemsep=0.2em, topsep=0.2em]
    \item \textbf{Empirical validation of robustness via worst-case elevation:} while standard RLVR improves mean accuracy, PerMix-RLVR is significantly more effective at bolstering worst-case performance across various personas.
    \item \textbf{Consistent effectiveness under domain shift:} PerMix-RLVR demonstrates sustained effectiveness even in shifted domains, maintaining persona stability at a level comparable to distillation methods that rely on much larger teacher models.
    \item \textbf{Persona-mixed RLVR offers the best overall balance:} beyond gains in accuracy and persona stability, PerMix-RLVR delivers the most significant improvements in persona expressivity and fidelity.
\end{enumerate}
\end{takeawaybox}

\section{Conclusion}
Persona prompting at inference time can induce large and unpredictable performance variance, turning prompt selection into a costly \textit{persona lottery}. In this work, we show that RLVR training improves persona robustness on tasks with verifiable goals, but also exposes a trade-off: by favoring styles more compatible with correct reasoning, it can suppress persona-induced styles that are less outcome-aligned, thereby attenuating persona expressivity and fidelity when faithful role-playing is required.

To address this limitation, we propose \textbf{PerMix-RLVR}, a persona-mixed RLVR strategy that learns the same verifiable-reward update under diverse persona-shifted priors. Across math and code domains, \textbf{PerMix-RLVR} consistently improves worst-case performance and robustness across various personas. Meanwhile, it better preserves persona fidelity, yielding a stronger overall balance among task accuracy, persona stability, and persona expressivity. Our work suggests that training-time persona conditioning is a promising direction for building models that remain robust to harmful persona variation without sacrificing controllable persona adoption when it is desired.

\bibliography{colm2026_conference}
\bibliographystyle{colm2026_conference}
\appendix
\clearpage

\section{Theoretical Analysis: Full Derivations}
\label{app:theory}

This appendix provides the full modeling framework and derivations for the results stated in Section~\ref{sec:theory}.

\subsection{Generative Model and Assumptions}
\label{app:model}

We model the generation process with the graphical model as shown in Figure~\ref{fig:graphical_model}.

\begin{figure}[!h]
\centering
\small
\begin{tikzpicture}[
    node distance=1.6cm,
    latent/.style={circle, draw, minimum size=1cm, inner sep=0pt},
    observed/.style={circle, draw, fill=gray!25, minimum size=1cm, inner sep=0pt},
    optional/.style={circle, draw, dashed, fill=gray!10, minimum size=1cm, inner sep=0pt},
    every edge/.style={draw, ->, >=stealth, thick},
]
    \node[observed]  (x) {$x$};
    \node[optional]  (p) [left=of x] {$p$};
    \node[latent]    (s) [above right=0.8cm and 1.2cm of p] {$s$};
    \node[observed]  (z) [right=of s] {$z$};
    \node[observed]  (V) [right=of z] {$V$};

    \draw[->, dashed] (p) -- (s);
    \draw[->] (x) -- (s);
    \draw[->] (s) -- (z);
    \draw[->] (x) to[bend right=20] (z);
    \draw[->] (z) -- (V);
\end{tikzpicture}
\caption{Graphical model of the generation process.
Shaded nodes are observed; the unshaded node $s$ (style) is a latent variable introduced for analysis.
The persona prompt $p$ (dashed) is optional: it is absent in standard RLVR and present in persona-conditioned RLVR.
The reasoning trajectory $z$ is the generated token sequence; the answer $a(z)$ and verifier output $V(z)$ are determined from $z$.
Note that $p$ influences $z$ only through $s$: once the style is fixed, $\pi(z|s,x)$ does not depend on $p$, which is why the style competence $\mu(s,x)$ is persona-independent.}
\label{fig:graphical_model}
\end{figure}
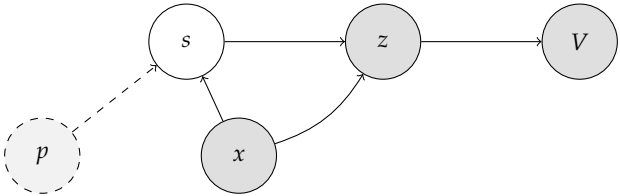

The model decomposes as follows.

\paragraph{Variables.}
\begin{itemize}[itemsep=0.2em]
    \item $x$: problem instance (observed).
    \item $p$: persona prompt (observed when present; absent in standard RLVR).
    \item $s$: style variable (latent)---an interpretive construct capturing tone, vocabulary, formatting, and persona adherence.
    \item $z$: reasoning trajectory (observed)---the generated token sequence from which the answer $a(z)$ is extracted.
    \item $V(z) \in \{0,1\}$: verifier output, determined by $z$ alone.
\end{itemize}

\paragraph{Factorization.}
The joint policy over the output $y = (s, z)$ factorizes as:
\begin{equation}
\label{eq:factorization_app}
    \pi(y \mid c) = \pi(s \mid c) \cdot \pi(z \mid s, x),
\end{equation}
where $c$ is the conditioning context ($c = x$ for standard RLVR, $c = (x,p)$ for persona-conditioned).

\paragraph{Assumptions and their roles.}

\begin{enumerate}[label=\textbf{(A\arabic*)}, leftmargin=2.5em, itemsep=0.3em]
    \item \textbf{Style--reasoning decomposition:} The output decomposes into a style component $s$ and a reasoning component $z$, with the factorization in Eq.~\eqref{eq:factorization_app}. This is an interpretive modeling assumption; in practice, style and reasoning are entangled in LLM outputs. However, the decomposition captures the key intuition that persona prompts primarily affect \emph{how} the model responds (style), not \emph{what} it reasons about (content).

    \item \textbf{Persona affects only style:} The persona $p$ enters only through the style prior $\pi(s|c)$. Once the style $s$ is fixed, the reasoning distribution $\pi(z|s,x)$ does not depend on $p$. In graphical model terms (Figure~\ref{fig:graphical_model}), $s$ d-separates $p$ from $z$ given $x$. This implies that the style competence $\mu(s,x) = \pi_{\mathrm{ref}}(\mathcal{C}(x)|s,x)$ is persona-independent---the central property driving the robustness analysis.

    \item \textbf{Reward depends only on reasoning:} $V(y) = V(z)$. The verifier is style-blind: it evaluates whether the answer is correct, not whether the response matches the persona's tone. This holds by construction for math and code verifiers.
\end{enumerate}

\paragraph{Scope.}
Assumption (A1) is used to derive the style-level filtering mechanism (Eq.~\ref{eq:style_posterior}) and the qualitative analysis of miscalibration vs.\ calibration in Section~\ref{sec:theory}.
Proposition~\ref{prop:acc_robust} does \emph{not} require (A1): it depends only on (A3) and the structure of the KL-regularized optimal policy.

\subsection{Optimal Policy}

RLVR maximizes $J(\pi) = \mathbb{E}_{y \sim \pi(\cdot|c)}[V(z)] - \beta\, D_{\mathrm{KL}}(\pi(\cdot|c) \| \pi_{\mathrm{ref}}(\cdot|c))$, where $c$ is the conditioning context.
$J$ is strictly concave. Setting the functional derivative to zero gives:
\begin{equation}
\label{eq:optimal_app}
    \pi^*(y \mid c) = \frac{1}{Z(c)}\,\pi_{\mathrm{ref}}(y \mid c)\,\exp\!\Big(\frac{V(z)}{\beta}\Big), \qquad Z(c) = \int \pi_{\mathrm{ref}}(y|c)\,\exp(V(z)/\beta)\,dy.
\end{equation}

\subsection{Derivation of the Style Posterior (Eq.~\ref{eq:style_posterior})}
\label{app:style_posterior}

Using the factorization $\pi(y|c) = \pi(s|c) \cdot \pi(z|s,x)$ with $V$ depending only on $z$, the marginal over $z$ gives:
\begin{equation}
    \pi^*(s \mid c) = \frac{1}{Z(c)}\,\pi_{\mathrm{ref}}(s \mid c) \cdot K(s, x), \qquad K(s, x) \equiv \mathbb{E}_{z \sim \pi_{\mathrm{ref}}(\cdot|s,x)}\!\left[\exp\!\Big(\frac{V(z)}{\beta}\Big)\right].
\end{equation}
Since $V(z) \in \{0,1\}$, the style competence function simplifies to:
\begin{equation}
    K(s, x) = e^{1/\beta}\,\mu(s, x) + (1 - \mu(s, x)) = 1 + (e^{1/\beta} - 1)\,\mu(s, x),
\end{equation}
where $\mu(s, x) = \pi_{\mathrm{ref}}(\mathcal{C}(x) \mid s, x)$.
Substituting back yields Eq.~\eqref{eq:style_posterior}.

\paragraph{Properties of the competence term.}
\begin{itemize}
    \item $K(s,x) \geq 1$ for all $s$, with equality when $\mu(s,x) = 0$.
    \item $K(s,x) = e^{1/\beta}$ when $\mu(s,x) = 1$.
    \item As $\beta \to 0$, the ratio between any competent style ($\mu > 0$) and an incompetent one ($\mu = 0$) grows without bound, concentrating the posterior on competent styles.
\end{itemize}

\paragraph{Why $\mu(s,x)$ is persona-independent.}
In the factorization $\pi(y|c) = \pi(s|c) \cdot \pi(z|s,x)$, the persona $p$ enters only through $\pi(s|c)$.
Once $s$ is fixed, $\pi(z|s,x)$ does not depend on which persona induced $s$, so $\mu(s,x) = \pi_{\mathrm{ref}}(\mathcal{C}(x)|s,x)$ is a property of the style--problem pair alone.

\subsection{Connection Between Style Filtering and Trajectory-Level Accuracy}

The partition function decomposes as:
\begin{equation}
    Z(c) = \mathbb{E}_{s \sim \pi_{\mathrm{ref}}(\cdot|c)}\!\big[K(s,x)\big] = 1 + (e^{1/\beta} - 1)\,\underbrace{\mathbb{E}_{s \sim \pi_{\mathrm{ref}}(\cdot|c)}\!\big[\mu(s,x)\big]}_{= \mu_p}.
\end{equation}
The last equality follows from the law of total probability: $\mu_p = \pi_{\mathrm{ref}}(\mathcal{C}(x)|c) = \mathbb{E}_s[\mu(s,x)]$.
This recovers $Z(c) = e^{1/\beta}\mu_p + (1-\mu_p)$, linking the style-level competence to the trajectory-level accuracy $\alpha(\mu_p, \beta)$.

\subsection{Proof of Proposition~\ref{prop:acc_robust}}
\label{app:proof_proposition}

\paragraph{(i) Accuracy.}
Let $\mu_p = \pi_{\mathrm{ref}}(\mathcal{C}(x)|x,p)$. Since $V(z) \in \{0,1\}$: $Z(x,p) = e^{1/\beta}\mu_p + (1-\mu_p)$.
\begin{equation}
    \mathrm{Acc}(\pi^*, x, p) = \frac{e^{1/\beta}\mu_p}{Z(x,p)} = \frac{e^{1/\beta}\mu_p}{e^{1/\beta}\mu_p + (1-\mu_p)} \equiv \alpha(\mu_p, \beta).
\end{equation}

\paragraph{(ii) Improvement.}
\begin{equation}
    \alpha(\mu, \beta) - \mu = \frac{\mu(1-\mu)(e^{1/\beta} - 1)}{e^{1/\beta}\mu + (1-\mu)} > 0 \quad \text{for } \mu \in (0,1),\; \beta > 0.
\end{equation}

\paragraph{(iii) Robustness (PSS bound).}
Let $\mu_{\min} = \min_p \mu_p$, $\mu_{\max} = \max_p \mu_p$.
Define $h(\mu) = \alpha(\mu)/\mu = e^{1/\beta}/(e^{1/\beta}\mu + 1 - \mu)$. Then $h'(\mu) < 0$, so $h(\mu_{\min}) \geq h(\mu_{\max})$:
\begin{equation}
    \PSS(\pi^*) = \frac{\alpha(\mu_{\min})}{\alpha(\mu_{\max})} \geq \frac{\mu_{\min}}{\mu_{\max}} = \PSS(\pi_{\mathrm{ref}}).
\end{equation}

\paragraph{Zero-temperature limit.}
As $\beta \to 0$: $\alpha(\mu, \beta) \to 1$ for any $\mu > 0$, so $\PSS \to 1$.

\paragraph{Concavity interpretation.}
$\alpha'(\mu) = e^{1/\beta}/(e^{1/\beta}\mu + 1 - \mu)^2$ is decreasing in $\mu$: the same absolute improvement in $\mu$ produces a larger accuracy gain at low $\mu_p$ than at high $\mu_p$, compressing the persona-induced performance gap.

\subsection{Scope of the Style Factorization}
The factorization $\pi(y \mid c) = \pi(s \mid c)\,\pi(z \mid s,x)$ is a stylized modeling assumption rather than a literal claim about how LLMs internally separate style from reasoning.
We adopt it as a tractable abstraction for verifiable reasoning tasks, where the reward depends on answer correctness rather than persona fidelity.
Under this abstraction, persona prompts affect performance by shifting the model toward styles that differ in their compatibility with correct reasoning.
In practice, style and reasoning may be more richly entangled, so the analysis should be interpreted as identifying one useful mechanism for persona robustness rather than a complete account of persona-conditioned generation.
Proposition~\ref{prop:acc_robust} does not rely on this factorization; it depends only on $V(z) \in \{0,1\}$ and the KL-regularized optimal policy.
\clearpage
\section{Implementation Details}
\label{app:sec:implementation}

\subsection{Basic Experimental Setup}

We utilized a unified base model and computing platform to ensure a fair comparison across all post-training paradigms (SFT, KD, and RLVR).

\begin{itemize}
    \item \textbf{Base Model:} Llama-3.1-8B-Instruct (\href{https://huggingface.co/meta-llama/Llama-3.1-8B-Instruct}{meta-llama/Llama-3.1-8B-Instruct})
    \item \textbf{Hardware:} 8$\times$ NVIDIA A5000 GPUs (24GB VRAM each)
    \item \textbf{Distributed Strategy:} Distributed Data Parallel (DDP) via \texttt{accelerate} library
    \item \textbf{Optimization Technique:} QLoRA (4-bit quantization)
    \item \textbf{Software Stack}
    \begin{table}[h]
        \centering
        \small
        \label{tab:hyperparams}
        \begin{tabular}{lc}
        \toprule
        \textbf{Name} & \textbf{Version} \\ \midrule
        \texttt{torch} & 2.9.1 \\
        \texttt{transformers} & 4.57.6 \\
        \texttt{accelerate} & 1.12.0 \\
        \texttt{peft} & 0.18.1 \\
        \texttt{CUDA toolkit} & 11.8 \\
        \bottomrule
        \end{tabular}
    \end{table}
\end{itemize}

To perform efficient training within constrained computing resources, we employed QLoRA. The specific configurations are as follows:

\begin{itemize}
    \item \textbf{Quantization:} 4-bit NormalFloat (NF4) with double quantization.
    \item \textbf{LoRA Configuration:} Rank ($r=16$), Alpha ($\alpha=32$), and Dropout (0.05).
    \item \textbf{Target Modules:} The LoRA adapters were applied to all linear layers, including \texttt{q\_proj}, \texttt{k\_proj}, \texttt{v\_proj}, \texttt{o\_proj}, \texttt{gate\_proj}, \texttt{up\_proj}, and \texttt{down\_proj}.
\end{itemize}

\subsection{SFT}

We conducted SFT using two primary configurations: \textbf{Non-Persona SFT} and \textbf{PerMix-SFT}, based on the prompt construction described in Appendix~\ref{app:sec:prompt_construction}. The models were trained to minimize the Cross-Entropy loss:

\begin{equation}
\mathcal{L}_{CE}(\theta) = -\sum_{t=1}^{|y|} \log P_{\theta}(y_t | x, y_{<t})
\end{equation}

For the target sequence $y$, we utilized:
\begin{itemize}
    \item \textbf{Ground-Truth (GT):} The original rationales and answers from the GSM8K training set.
    \item \textbf{Persona-Consistent Style ($y_s$):} Target responses rewritten by GPT-5-mini to align with the sampled persona's linguistic style, ensuring the model learns to maintain persona fidelity while reasoning.
\end{itemize}

\begin{table}[h]
\centering
\small
\begin{tabular}{lc}
\hline
\textbf{Hyperparameter} & \textbf{Value} \\ \hline
Optimizer & PagedAdamW (8bit) \\
Data Type & bfloat16 \\
Learning Rate & $1 \times 10^{-5}$ \\
LR Scheduler & cosine \\
Batch Size & 128 \\
Max Sequence Length & 1024 \small{(Non-Persona)} / 2048 \small{(PerMix-SFT)} \\
Epochs & 3 \\
Warmup Ratio & 0.05 \\
Weight Decaying & 0.0 \\
\hline
\end{tabular}
\caption{Hyperparameters for SFT}
\label{tab:hyperparams_sft}
\end{table}

Detailed configuration parameters used for SFT are provided in Table~\ref{tab:hyperparams_sft}.

\subsection{SeqKD}

\begin{table}[h]
\centering
\small
\begin{tabular}{lc}
\hline
\textbf{Hyperparameter} & \textbf{Value} \\ \hline
Optimizer & PagedAdamW (8bit) \\
Data Type & bfloat16 \\
Learning Rate & $2 \times 10^{-5}$ \\
LR Scheduler & cosine \\
Batch Size & 128 \\
Max Sequence Length & 1024 \\
Epochs & 3 \\
Warmup Ratio & 0.05 \\
Weight Decaying & 0.0 \\
\hline
\end{tabular}
\caption{Hyperparameters for SeqKD}
\label{tab:hyperparams_seqkd}
\end{table}

The loss function for \textbf{SeqKD} \citep{kim2016sequence} can be derived as follows:

\begin{align*}
    \mathcal{L}_{SeqKD} &= -\sum_{t \in \mathcal{T}} q(t|s) \log p(t | s) \\
    q(t|s) &\approx \mathbb{1}[t = \mathop{\arg\max}_{t \in \mathcal{T}}q(t|s)] \\
    \mathcal{L}_{SeqKD} &\approx - \sum_{t \in \mathcal{T}} \mathbb{1}[t = \hat{y}] \log p(t|s) \\
    &= - \log p(t=\hat{y}|s)
\end{align*}

where $q(t|s)$ represents the teacher's sequence distribution over the sample space of all possible sequences, and $\hat{y}$ is the output mode from the teacher model. As shown in the final form of the equation, the objective is equivalent to the cross-entropy loss ($\mathcal{L}_{CE}$).

Similarly to the SFT approach, we designed two variations for SeqKD: Non-Persona and Persona-mixed.
\begin{itemize}
    \item \textbf{Non-Persona Conditioned Teacher Response ($y^t$):} The teacher output generated given only the task instruction.
    \item \textbf{Persona Conditioned Teacher Response ($y^t_p$):} The teacher output extracted when the instruction is conditioned by a randomly sampled persona $p \sim \mathcal{P}_{\text{train}}$.
\end{itemize}

The specific configuration parameters and hyperparameters used for SeqKD are provided in Table~\ref{tab:hyperparams_seqkd}.

\subsection{RLVR}

For the RLVR implementation, we optimize the model using the objective defined in Equation~\ref{eq:rlvr_objective}, adopting the Group Relative Policy Optimization (GRPO) framework.

To evaluate the reasoning trajectories, we utilize the \texttt{math-verify} library provided by HuggingFace to extract final answers from the output sequences. The verifier employs a binary correctness reward, denoted as $r = \mathbb{1}[V(x,y)=a]$, which assigns a reward of 1.0 if the extracted answer matches the ground truth and 0 otherwise.

\begin{table}[h]
\centering
\small
\begin{tabular}{lc}
\hline
\textbf{Hyperparameter} & \textbf{Value} \\ \hline
Optimizer & Fused AdamW \\
Data Type & bfloat16 \\
Learning Rate & $1 \times 10^{-5}$ \\
LR Scheduler & cosine \\
Batch Size & 64 \\
Max Prompt Length & 256 \\
Max Completion Length & 512 \\
Number of Generations & 8 \\
Epochs & 1 \\
Warmup Ratio & 0.05 \\
Weight Decaying & 0.1 \\
Maximum Gradient Norm & 0.1 \\
\hline
\end{tabular}
\caption{Hyperparameters for RLVR}
\label{tab:hyperparams_rlvr}
\end{table}

The specific configuration parameters and hyperparameters used for RLVR training are detailed in Table~\ref{tab:hyperparams_rlvr}.

\subsection{PerMix-RLVR}

Algorithm~\ref{alg:permix-rlvr} outlines the whole training process for \textbf{PerMix-RLVR} implementation. Except for the specialized input formatting, the method follows the identical process as standard RLVR. Additionally, \textbf{PerMix-RLVR} shared the same hyperparameter configurations as the standard RLVR training, as detailed in Table~\ref{tab:hyperparams_rlvr}.

\begin{algorithm}[!h]
\caption{PerMix-RLVR: Persona-Mixed RL with Verifiable Rewards}
\label{alg:permix-rlvr}
\DontPrintSemicolon
\KwIn{Training data $\mathcal{D}=\{(x_i,a_i)\}$, persona pool $\mathcal{P}_{\text{train}}$, system template $s(\cdot)$, verifier $V(\cdot)$, policy $\pi_{\theta}$, reference policy $\pi_{\text{ref}}$, KL coefficient $\beta$, group size $K$.}
\KwOut{Updated policy parameters $\theta$.}

\For{each training step}{
    Sample a minibatch $\mathcal{B}\subset\mathcal{D}$\;
    \ForEach{$(x,a)\in \mathcal{B}$}{
        Sample persona $p \sim \mathrm{Unif}(\mathcal{P}_{\text{train}})$\;
        Construct prompt $u \leftarrow (s(p), x)$ \tcp*[r]{personas in system prompt}
        Sample $K$ responses $\{y^{(k)}\}_{k=1}^K \sim \pi_{\theta}(\cdot \mid u)$\;
        \For{$k=1$ \KwTo $K$}{
            $r^{(k)} \leftarrow \mathbb{1}\!\left[V(x, y^{(k)}) = a\right]$ \tcp*[r]{verifiable reward}
            $d^{(k)} \leftarrow \log \pi_{\theta}(y^{(k)}\mid u) - \log \pi_{\text{ref}}(y^{(k)}\mid u)$ \tcp*[r]{KL term}
        }
        Compute advantages $\{\hat{A}^{(k)}\}$ from $\{r^{(k)}\}$ (e.g., group-normalized)\;
        Update $\theta$ by maximizing:
        \[
        \mathbb{E}_{k}\!\left[\hat{A}^{(k)} \cdot \log \pi_{\theta}(y^{(k)}\mid u)\right]
        \;-\; \beta \cdot \mathbb{E}_{k}\!\left[d^{(k)}\right].
        \]
    }
}
\Return{$\theta$}\;
\end{algorithm}

\section{Additional Results}
\label{app:sec:ablation_study}

\subsection{PSS comparison across diverse set of models and datasets}
We report the full persona-wise results underlying Figure~\ref{fig:pss_graph} (see Tables~\ref{app:tab:math500_persona_by_model} and \ref{app:tab:aime2024_persona_by_model} for the exact evaluation prompts). For each model, we evaluate 16 persona prompts and run each persona-condition 5 times, reporting mean ± std accuracy per persona. We additionally compute Persona Sensitivity Score (PSS) to summarize how stable a model’s performance is across personas (higher is more stable). Across both MATH500 and AIME2024, we observe a clear family-level ordering in persona robustness: Qwen3 consistently achieves the highest PSS, followed by Gemma3, while Llama3 exhibits the lowest PSS, indicating the largest performance disparities across persona prompts. The gap becomes more pronounced on the harder AIME2024 benchmark, where some smaller models show near-floor accuracies for several personas, amplifying variance and lowering PSS. Finally, beyond family effects, scaling generally improves stability within a family—larger checkpoints tend to yield higher PSS and smaller per-persona standard deviations—supporting the view that both model family and model scale jointly govern robustness to persona-induced shifts.

\begin{table*}[t!]
\centering
\small
\addtolength{\tabcolsep}{-3pt}
\resizebox{\textwidth}{!}{%
\begin{tabular}{l|ccccc|ccc|cccc}
\toprule
& \multicolumn{5}{c|}{\textbf{Qwen3}} 
& \multicolumn{3}{c|}{\textbf{Llama3}} 
& \multicolumn{4}{c}{\textbf{Gemma3}} \\
\cmidrule(lr){2-6} \cmidrule(lr){7-9} \cmidrule(lr){10-13}
\textbf{Persona}
& \textbf{0.6B} & \textbf{1.7B} & \textbf{4B} & \textbf{8B} & \textbf{32B}
& \textbf{3.2-1B} & \textbf{3.1-8B} & \textbf{3.3-70B}
& \textbf{1B} & \textbf{4B} & \textbf{12B} & \textbf{27B} \\
\midrule
\texttt{math expert} 
& 47.32$_{\,\pm\,1.66}$ & 69.76$_{\,\pm\,1.04}$ & 79.28$_{\,\pm\,0.66}$ & 80.28$_{\,\pm\,1.27}$ & 82.08$_{\,\pm\,1.38}$
& 14.80$_{\,\pm\,2.63}$ & 43.88$_{\,\pm\,1.00}$ & 71.76$_{\,\pm\,1.15}$
& 26.42$_{\,\pm\,1.67}$ & 70.40$_{\,\pm\,0.93}$ & 80.92$_{\,\pm\,0.40}$ & 86.00$_{\,\pm\,0.84}$ \\
\texttt{software engineer} & 46.28$_{\,\pm\,1.99}$ & 68.68$_{\,\pm\,1.17}$ & 77.48$_{\,\pm\,0.97}$ & 80.12$_{\,\pm\,1.08}$ & 82.80$_{\,\pm\,0.96}$ & 15.72$_{\,\pm\,1.71}$ & 44.8$_{\,\pm\,1.79}$ & 72.44$_{\,\pm\,1.02}$ & 20.92$_{\,\pm\,2.13}$ & 69.24$_{\,\pm\,1.13}$ & 78.92$_{\,\pm\,1.14}$ & 85.16$_{\,\pm\,0.71}$ \\
\texttt{physics professor} & 46.92$_{\,\pm\,1.24}$ & 69.80$_{\,\pm\,1.85}$ & 78.84$_{\,\pm\,1.14}$  & 80.12$_{\,\pm\,0.52}$ & 82.48$_{\,\pm\,0.44}$ & 15.44$_{\,\pm\,0.73}$ & 43.76$_{\,\pm\,0.93}$ & 72.44$_{\,\pm\,1.11}$ & 22.64$_{\,\pm\,1.87}$ & 66.8$_{\,\pm\,0.68}$ & 78.72$_{\,\pm\,1.40}$ & 84.09$_{\,\pm\,0.81}$ \\
\texttt{semiconductor specialist} & 46.64$_{\,\pm\,0.95}$ & 70.00$_{\,\pm\,1.36}$ & 77.92$_{\,\pm\,1.15}$ & 81.60$_{\,\pm\,0.58}$ & 81.84$_{\,\pm\,0.43}$ & 16.16$_{\,\pm\,1.72}$ & 45.00$_{\,\pm\,1.20}$ & 71.28$_{\,\pm\,0.90}$ & 22.08$_{\,\pm\,1.32}$ & 64.72$_{\,\pm\,1.14}$ & 77.56$_{\,\pm\,1.72}$ & 82.48$_{\,\pm\,0.67}$ \\
\midrule
\texttt{kindergarten} 
& 45.48$_{\,\pm\,0.87}$ & 68.08$_{\,\pm\,1.20}$ & 78.16$_{\,\pm\,1.08}$ & 78.24$_{\,\pm\,0.55}$ & 83.84$_{\,\pm\,1.65}$
& 9.16$_{\,\pm\,1.28}$ & 32.16$_{\,\pm\,1.99}$ & 70.76$_{\,\pm\,0.91}$
& 25.00$_{\,\pm\,1.15}$ & 49.56$_{\,\pm\,1.32}$ & 70.20$_{\,\pm\,0.68}$ & 78.88$_{\,\pm\,1.00}$ \\
\texttt{high school} & 46.12$_{\,\pm\,1.38}$ & 68.56$_{\,\pm\,1.37}$ & 79.12$_{\,\pm\,0.92}$ & 80.64$_{\,\pm\,0.98}$ & 82.32$_{\,\pm\,0.73}$ & 12.28$_{\,\pm\,2.39}$ & 43.52$_{\,\pm\,1.21}$ & 72.60$_{\,\pm\,1.06}$ & 21.84$_{\,\pm\,0.91}$ & 67.76$_{\,\pm\,1.15}$ & 78.72$_{\,\pm\,1.14}$ & 84.00$_{\,\pm\,1.32}$ \\
\texttt{undergraduate} & 47.64$_{\,\pm\,1.57}$ & 69.52$_{\,\pm\,1.74}$ & 78.84$_{\,\pm\,1.09}$ & 81.24$_{\,\pm\,1.56}$ & 82.16$_{\,\pm\,0.65}$ & 16.32$_{\,\pm\,1.20}$ & 43.72$_{\,\pm\,2.07}$ & 71.72$_{\,\pm\,0.41}$ & 21.36$_{\,\pm\,1.77}$ & 68.24$_{\,\pm\,1.66}$ & 79.52$_{\,\pm\,1.00}$ & 85.00$_{\,\pm\,1.16}$ \\
\texttt{phd graduate} & 49.24$_{\,\pm\,1.61}$ & 69.44$_{\,\pm\,1.77}$ & 78.96$_{\,\pm\,0.91}$ & 80.84$_{\,\pm\,1.09}$ & 83.08$_{\,\pm\,0.89}$ & 18.32$_{\,\pm\,1.46}$ & 43.84$_{\,\pm\,1.33}$ & 72.60$_{\,\pm\,1.21}$ & 22.48$_{\,\pm\,1.51}$ & 67.32$_{\,\pm\,0.94}$ & 78.80$_{\,\pm\,0.71}$ & 84.16$_{\,\pm\,1.17}$ \\
\midrule
\texttt{clever} & 47.32$_{\,\pm\,0.99}$ & 69.12$_{\,\pm\,1.38}$ & 78.2$_{\,\pm\,1.42}$ & 79.88$_{\,\pm\,2.13}$ & 81.60$_{\,\pm\,1.99}$ & 16.16$_{\,\pm\,0.79}$ & 42.6$_{\,\pm\,1.39}$ & 72.20$_{\,\pm\,1.09}$ & 16.76$_{\,\pm\,1.90}$ & 64.04$_{\,\pm\,1.56}$ & 77.92$_{\,\pm\,0.52}$ & 83.56$_{\,\pm\,1.24}$ \\
\texttt{dumb} & 43.52$_{\,\pm\,1.73}$ & 68.32$_{\,\pm\,0.87}$ & 79.24$_{\,\pm\,1.81}$ & 79.12$_{\,\pm\,0.74}$ & 81.88$_{\,\pm\,1.00}$ & 10.64$_{\,\pm\,1.11}$ & 35.92$_{\,\pm\,1.31}$ & 67.48$_{\,\pm\,1.25}$ & 20.92$_{\,\pm\,1.60}$ & 68.36$_{\,\pm\,0.75}$ & 75.12$_{\,\pm\,1.29}$ & 84.08$_{\,\pm\,0.84}$ \\
\texttt{questioning} & 48.04$_{\,\pm\,1.87}$ & 69.12$_{\,\pm\,0.94}$ & 78.36$_{\,\pm\,0.74}$ & 81.20$_{\,\pm\,1.03}$ & 82.32$_{\,\pm\,0.94}$ & 16.28$_{\,\pm\,0.73}$ & 43.52$_{\,\pm\,2.27}$ & 71.32$_{\,\pm\,1.71}$ & 23.16$_{\,\pm\,1.56}$ & 65.88$_{\,\pm\,0.41}$ & 79.80$_{\,\pm\,1.21}$ & 83.96$_{\,\pm\,0.79}$ \\
\texttt{easygoing} & 46.80$_{\,\pm\,1.30}$ & 69.20$_{\,\pm\,0.73}$ & 79.56$_{\,\pm\,1.19}$ & 79.60$_{\,\pm\,0.84}$ & 82.28$_{\,\pm\,0.81}$ & 13.04$_{\,\pm\,1.39}$ & 44.12$_{\,\pm\,1.29}$ & 71.80$_{\,\pm\,0.47}$ & 22.96$_{\,\pm\,2.30}$ & 64.76$_{\,\pm\,1.57}$ & 76.72$_{\,\pm\,0.79}$ & 82.44$_{\,\pm\,1.06}$ \\
\midrule
\texttt{carpenter} 
& 46.20$_{\,\pm\,1.61}$ & 69.08$_{\,\pm\,2.27}$ & 78.52$_{\,\pm\,0.72}$ & 80.66$_{\,\pm\,1.05}$ & 81.64$_{\,\pm\,0.99}$
& 13.40$_{\,\pm\,2.09}$ & 42.96$_{\,\pm\,1.85}$ & 72.60$_{\,\pm\,0.68}$
& 31.64$_{\,\pm\,0.55}$ & 68.08$_{\,\pm\,0.92}$ & 78.80$_{\,\pm\,1.08}$ & 84.16$_{\,\pm\,0.79}$ \\
\texttt{teacher} & 46.40$_{\,\pm\,1.84}$ & 69.88$_{\,\pm\,1.10}$ & 78.76$_{\,\pm\,1.16}$ & 80.16$_{\,\pm\,0.84}$ & 82.64$_{\,\pm\,0.26}$ & 17.08$_{\,\pm\,1.36}$ & 44.36$_{\,\pm\,0.86}$ & 71.64$_{\,\pm\,1.68}$ & 15.72$_{\,\pm\,1.89}$ & 62.96$_{\,\pm\,0.89}$ & 76.80$_{\,\pm\,1.33}$ & 83.88$_{\,\pm\,0.77}$ \\
\texttt{lawyer} & 46.72$_{\,\pm\,1.38}$ & 70.44$_{\,\pm\,1.51}$ & 79.00$_{\,\pm\,0.66}$ & 80.48$_{\,\pm\,0.95}$ & 82.88$_{\,\pm\,0.92}$ & 15.76$_{\,\pm\,0.93}$ & 44.12$_{\,\pm\,2.28}$ & 71.76$_{\,\pm\,1.08}$ & 21.40$_{\,\pm\,0.58}$ & 65.88$_{\,\pm\,1.31}$ & 78.88$_{\,\pm\,0.58}$ & 84.56$_{\,\pm\,0.83}$ \\
\texttt{sports player} & 46.36$_{\,\pm\,0.67}$ & 69.44$_{\,\pm\,1.35}$ & 78.20$_{\,\pm\,0.79}$ & 80.04$_{\,\pm\,1.19}$ & 82.12$_{\,\pm\,1.04}$ & 15.60$_{\,\pm\,0.93}$ & 42.56$_{\,\pm\,0.71}$ & 72.12$_{\,\pm\,1.00}$ & 21.40$_{\,\pm\,0.58}$ & 63.96$_{\,\pm\,0.85}$ & 77.32$_{\,\pm\,0.52}$ & 82.92$_{\,\pm\,1.07}$ \\
\midrule
\texttt{PSS} & 0.8838 & 0.9665 & 0.9739 & 0.9588 & 0.9733 & 0.5000 & 0.7147 & 0.9295 & 0.4968 & 0.7040 & 0.8675 & 0.9172 \\
\bottomrule
\end{tabular}%
}
\caption{Persona-wise performance on \textbf{MATH500}. Columns are grouped by model family (Qwen/Llama/Gemma). Values are mean$_{\,\pm\,\mathrm{std}}$.}
\label{app:tab:math500_persona_by_model}
\end{table*}

\begin{table*}[t!]
\centering
\small
\addtolength{\tabcolsep}{-3pt}
\resizebox{\textwidth}{!}{%
\begin{tabular}{l|ccccc|ccc|cccc}
\toprule
& \multicolumn{5}{c|}{\textbf{Qwen3}} 
& \multicolumn{3}{c|}{\textbf{Llama3}} 
& \multicolumn{4}{c}{\textbf{Gemma3}} \\
\cmidrule(lr){2-6} \cmidrule(lr){7-9} \cmidrule(lr){10-13}
\textbf{Persona}
& \textbf{0.6B} & \textbf{1.7B} & \textbf{4B} & \textbf{8B} & \textbf{32B}
& \textbf{3.2-1B} & \textbf{3.1-8B} & \textbf{3.3-70B}
& \textbf{1B} & \textbf{4B} & \textbf{12B} & \textbf{27B} \\
\midrule
\texttt{math expert} 
& 1.98$_{\,\pm\,1.81}$ & 12.62$_{\,\pm\,4.34}$ & 21.33$_{\,\pm\,3.80}$ & 28.00$_{\,\pm\,6.91}$ & 23.99$_{\,\pm\,8.94}$
& 0.66$_{\,\pm\,1.49}$ & 6.66$_{\,\pm\,2.79}$ & 24.00$_{\,\pm\,4.42}$
& 2.00$_{\,\pm\,2.98}$ & 7.33$_{\,\pm\,4.94}$ & 25.33$_{\,\pm\,5.05}$ & 27.33$_{\,\pm\,3.89}$ \\
\texttt{software engineer} & 1.33$_{\,\pm\,1.83}$ & 11.33$_{\,\pm\,2.98}$ & 19.33$_{\,\pm\,4.94}$ & 24.67$_{\,\pm\,3.80}$ & 24.00$_{\,\pm\,2.79}$ & 2.00$_{\,\pm\,2.98}$ & 3.33$_{\,\pm\,2.36}$ & 25.33$_{\,\pm\,6.86}$ & 0.67$_{\,\pm\,1.49}$ & 6.00$_{\,\pm\,3.65}$ & 20.00$_{\,\pm\,2.36}$ & 24.00$_{\,\pm\,2.50}$\\
\texttt{physics professor} & 1.33$_{\,\pm\,1.83}$ & 10.67$_{\,\pm\,2.79}$ & 18.67$_{\,\pm\,3.80}$ & 28.00$_{\,\pm\,4.47}$ & 24.67$_{\,\pm\,3.80}$ & 0.67$_{\,\pm\,1.49}$ & 4.00$_{\,\pm\,2.79}$ & 24.00$_{\,\pm\,3.27}$ & 0.67$_{\,\pm\,1.49}$ & 9.33$_{\,\pm\,6.41}$ & 22.67$_{\,\pm\,2.79}$ & 24.00$_{\,\pm\,2.50}$ \\
\texttt{semiconductor specialist} & 2.00$_{\,\pm\,2.98}$ & 10.67$_{\,\pm\,1.49}$ & 21.33$_{\,\pm\,1.83}$ & 24.00$_{\,\pm\,6.41}$ & 27.33$_{\,\pm\,2.79}$ & 0.00$_{\,\pm\,0.00}$ & 4.67$_{\,\pm\,1.83}$ & 24.67$_{\,\pm\,3.40}$ & 0.00$_{\,\pm\,0.00}$ & 4.67$_{\,\pm\,3.80}$ & 22.00$_{\,\pm\,5.06}$ & 18.00$_{\,\pm\,4.99}$ \\
\midrule
\texttt{kindergarten} 
& 1.98$_{\,\pm\,1.81}$ & 11.30$_{\,\pm\,2.99}$ & 18.66$_{\,\pm\,5.58}$ & 21.33$_{\,\pm\,2.98}$ & 27.30$_{\,\pm\,2.81}$
& 1.33$_{\,\pm\,1.82}$ & 1.99$_{\,\pm\,2.98}$ & 21.33$_{\,\pm\,1.63}$
& 0.67$_{\,\pm\,1.49}$ & 8.00$_{\,\pm\,5.06}$ & 16.66$_{\,\pm\,3.34}$ & 16.67$_{\,\pm\,3.65}$ \\
\texttt{high school} & 2.00$_{\,\pm\,1.83}$ & 6.67$_{\,\pm\,4.71}$ & 18.67$_{\,\pm\,3.80}$ & 26.67$_{\,\pm\,4.08}$ & 28.00$_{\,\pm\,6.06}$ & 0.00$_{\,\pm\,0.00}$ & 3.33$_{\,\pm\,3.33}$ & 24.67$_{\,\pm\,1.64}$ & 1.33$_{\,\pm\,1.83}$ & 11.33$_{\,\pm\,1.83}$ & 18.67$_{\,\pm\,2.98}$ & 26.67$_{\,\pm\,4.71}$ \\
\texttt{undergraduate} & 1.33$_{\,\pm\,1.83}$ & 10.67$_{\,\pm\,2.79}$ & 24.00$_{\,\pm\,4.35}$ & 26.67$_{\,\pm\,3.33}$ & 24.00$_{\,\pm\,8.30}$ & 2.00$_{\,\pm\,2.98}$ & 4.67$_{\,\pm\,2.98}$ & 20.00$_{\,\pm\,8.43}$ & 0.67$_{\,\pm\,1.49}$ & 9.33$_{\,\pm\,2.79}$ & 19.33$_{\,\pm\,2.79}$ & 26.00$_{\,\pm\,3.89}$ \\
\texttt{phd graduate} & 4.67$_{\,\pm\,2.98}$ & 8.00$_{\,\pm\,6.50}$ & 20.67$_{\,\pm\,2.79}$ & 20.67$_{\,\pm\,5.48}$ & 28.67$_{\,\pm\,1.83}$ & 1.33$_{\,\pm\,1.83}$ & 5.33$_{\,\pm\,3.80}$ & 24.67$_{\,\pm\,4.00}$ & 0.67$_{\,\pm\,1.49}$ & 9.33$_{\,\pm\,4.35}$ & 26.00$_{\,\pm\,7.23}$ & 28.67$_{\,\pm\,5.42}$ \\
\midrule
\texttt{clever} & 2.00$_{\,\pm\,1.83}$ & 11.33$_{\,\pm\,2.98}$ & 21.33$_{\,\pm\,9.01}$ & 26.67$_{\,\pm\,4.71}$ & 26.00$_{\,\pm\,5.48}$ & 0.00$_{\,\pm\,0.00}$ & 2.67$_{\,\pm\,1.49}$ & 28.00$_{\,\pm\,2.66}$ & 0.00$_{\,\pm\,0.00}$ & 8.00$_{\,\pm\,3.80}$ & 20.67$_{\,\pm\,4.35}$ & 23.33$_{\,\pm\,7.60}$ \\
\texttt{dumb} & 0.67$_{\,\pm\,1.49}$ & 10.67$_{\,\pm\,4.35}$ & 18.00$_{\,\pm\,1.83}$ & 21.33$_{\,\pm\,4.47}$ & 19.33$_{\,\pm\,5.96}$ & 0.67$_{\,\pm\,1.49}$ & 3.33$_{\,\pm\,0.00}$ & 22.00$_{\,\pm\,8.05}$ & 0.00$_{\,\pm\,0.00}$ & 6.00$_{\,\pm\,2.79}$ & 16.00$_{\,\pm\,4.35}$ & 22.00$_{\,\pm\,7.18}$ \\
\texttt{questioning} & 1.33$_{\,\pm\,1.83}$ & 12.00$_{\,\pm\,2.98}$ & 23.33$_{\,\pm\,5.77}$ & 22.00$_{\,\pm\,5.58}$ & 28.67$_{\,\pm\,1.83}$ & 2.67$_{\,\pm\,1.49}$ & 4.67$_{\,\pm\,2.98}$ & 25.67$_{\,\pm\,4.90}$ & 0.00$_{\,\pm\,0.00}$ & 9.33$_{\,\pm\,2.79}$ & 24.67$_{\,\pm\,4.47}$ & 22.67$_{\,\pm\,3.89}$ \\
\texttt{easygoing} & 3.33$_{\,\pm\,2.36}$ & 10.00$_{\,\pm\,4.08}$ & 20.00$_{\,\pm\,7.07}$ & 24.00$_{\,\pm\,6.41}$ & 28.67$_{\,\pm\,9.89}$ & 0.00$_{\,\pm\,0.00}$ & 6.00$_{\,\pm\,4.35}$ & 24.00$_{\,\pm\,2.50}$ & 1.33$_{\,\pm\,1.83}$ & 8.67$_{\,\pm\,1.83}$ & 20.00$_{\,\pm\,0.00}$ & 20.67$_{\,\pm\,4.42}$ \\
\midrule
\texttt{carpenter} 
& 3.96$_{\,\pm\,2.76}$ & 8.64$_{\,\pm\,1.86}$ & 19.33$_{\,\pm\,4.94}$ & 27.33$_{\,\pm\,3.65}$ & 24.66$_{\,\pm\,2.98}$
& 0.66$_{\,\pm\,1.49}$ & 3.33$_{\,\pm\,3.33}$ & 23.33$_{\,\pm\,2.98}$
& 1.33$_{\,\pm\,1.82}$ & 12.00$_{\,\pm\,3.80}$ & 22.00$_{\,\pm\,6.91}$ & 20.00$_{\,\pm\,7.60}$ \\
\texttt{teacher} & 1.33$_{\,\pm\,1.83}$ & 10.67$_{\,\pm\,4.35}$ & 20.00$_{\,\pm\,3.33}$ & 28.00$_{\,\pm\,5.06}$ & 25.33$_{\,\pm\,6.50}$ & 0.00$_{\,\pm\,0.00}$ & 3.33$_{\,\pm\,0.00}$ & 24.67$_{\,\pm\,3.40}$ & 0.67$_{\,\pm\,1.49}$ & 7.33$_{\,\pm\,2.79}$ & 18.67$_{\,\pm\,6.50}$ & 22.00$_{\,\pm\,3.40}$ \\
\texttt{lawyer} & 1.33$_{\,\pm\,1.83}$ & 11.33$_{\,\pm\,3.80}$ & 18.00$_{\,\pm\,4.47}$ & 26.67$_{\,\pm\,8.16}$ & 22.67$_{\,\pm\,4.35}$ & 0.00$_{\,\pm\,0.00}$ & 6.67$_{\,\pm\,2.36}$ & 22.67$_{\,\pm\,4.90}$ & 0.67$_{\,\pm\,1.49}$ & 10.00$_{\,\pm\,2.36}$ & 22.00$_{\,\pm\,3.80}$ & 24.00$_{\,\pm\,6.11}$ \\
\texttt{sports player} & 1.33$_{\,\pm\,1.83}$ & 12.67$_{\,\pm\,3.65}$ & 20.67$_{\,\pm\,6.83}$ & 27.33$_{\,\pm\,4.35}$ & 24.67$_{\,\pm\,8.37}$ & 0.67$_{\,\pm\,1.49}$ & 2.67$_{\,\pm\,2.79}$ & 22.00$_{\,\pm\,3.40}$ & 0.00$_{\,\pm\,0.00}$ & 8.00$_{\,\pm\,3.80}$ & 21.33$_{\,\pm\,5.06}$ & 19.34$_{\,\pm\,6.46}$ \\
\midrule
\texttt{PSS}
& 0.1435 & 0.5264 & 0.7500 & 0.7382 & 0.6742
& 0.0000 & 0.2984 & 0.7143
& 0.0000 & 0.3892 & 0.6154 & 0.5814 \\
\bottomrule
\end{tabular}%
}
\caption{Persona-wise performance on \textbf{AIME2024}. Rows are personas and columns are models. Values are mean$_{\,\pm\,\mathrm{std}}$.}
\label{app:tab:aime2024_persona_by_model}
\end{table*}

\subsection{PersonaGym}
\label{app:subsec:persona_gym_detail}
\paragraph{Experiments setup.}
 We evaluated all models using our evaluation persona pool (Table~\ref{app:tab:persona_prompts_eval}). Following \citet{samuel-etal-2025-personagym}, we employed two LLM judges—gpt-5-mini and DeepSeek-v3.2—to assess whether each model adhered to the assigned persona. Note that \citet{samuel-etal-2025-personagym} used gpt-4o and Llama-3-70B as judges; this choice can affect absolute scores. To standardize evaluation, gpt-5-mini generated a set of persona-relevant questions, and we used the same question set for each persona across all evaluated models. 
 
 In PersonaGym, persona-following is evaluated along five complementary axes: Expected Action (EA) checks whether the agent selects actions that are logically expected for the given persona in the specific environment; Action Justification (AJ) assesses whether the agent’s explanations and rationales are coherent and persona-appropriate; Linguistic Habits (LH) measures adherence to persona-typical communication patterns such as tone, jargon, and style; Toxicity Control (TC) evaluates whether the agent remains safe and non-toxic under persona-relevant provocative prompts; and crucially, Persona Consistency (PC) directly measures fidelity to the assigned persona attributes when queried, i.e., whether the agent consistently maintains the intended identity rather than reverting to generic assistant behavior. In our paper, this PC dimension most tightly matches what we mean by “role-playing”: it operationalizes the core requirement that a model stays in character (identity/traits/constraints) across interactions, making it the most aligned PersonaGym component for quantifying persona expressivity and faithfulness under persona prompts.
 
\paragraph{Additional results.}
We show additional results on PersonaGym in Table~\ref{app:tab:roleplay_detailed} aside from Persona Consistency. When aggregating all results, including Persona Consistency, we find that \textbf{PerMix-RLVR} achieves the best performance on PersonaGym.

\begin{table*}[!t]
\centering
\small
\resizebox{0.85\linewidth}{!}{
\begin{tabular}{l|cc|cccccc}
\toprule
\textbf{Post-training}
& \textbf{Input}
& \textbf{Objective ($\mathcal{J}$)}
& \texttt{EA}
& \texttt{AJ}
& \texttt{LH}
& \texttt{PC}
& \texttt{TC}
& \texttt{Mean Score} \\
\midrule

Base ($M_0$)
& $x$ & -
& 2.69 & 2.59 & 2.44 & 3.19 & 2.91 & 2.76 \\

\midrule
\multicolumn{9}{l}{\textbf{Supervised Fine-tuning}} \\
\midrule

SFT
& $x$ & $\mathcal{L}_{CE}(y)$
& 2.38 & 2.03 & 2.31 & 2.53 & 2.72 & 2.39 \\

\midrule

\multirow{2}{*}{PerMix-SFT} 
& \multirow{2}{*}{$x$+$p$} & $\mathcal{L}_{CE}(y)$
& 2.16 & 2.16 & 2.41 & 2.88 & 2.91 & 2.50 \\

&  & $\mathcal{L}_{CE}(y_s)$
& 2.56 & 2.59 & 2.25 & 3.31 & 3.13 & 2.77 \\

\midrule
\multicolumn{9}{l}{\textbf{Distillation}} \\
\midrule

\multirow{2}{*}{SeqKD-Qwen} 
& $x$ & $\mathcal{L}_{CE}(y^t)$
& 3.09 & 2.38 & 2.19 & 3.12 & 3.00 & 2.76 \\

& $x{+}p$ & $\mathcal{L}_{CE}(y^t_p)$
& 2.78 & 2.53 & 1.97 & 3.06 & 3.03 & 2.69 \\

\midrule

\multirow{2}{*}{SeqKD-Llama} 
& $x$ & $\mathcal{L}_{CE}(y^t)$
& 2.59 & 2.47 & 2.31 & 2.72 & 3.03 & 2.66 \\

& $x{+}p$ & $\mathcal{L}_{CE}(y^t_p)$
& 2.41 & 2.69 & 2.44 & 3.09 & 3.09 & 2.77 \\

\midrule

\multirow{2}{*}{SeqKD-Gemma} 
& $x$ & $\mathcal{L}_{CE}(y^t)$
& 2.66 & 2.78 & 2.47 & 3.06 & 3.03 & 2.82 \\

& $x{+}p$ & $\mathcal{L}_{CE}(y^t_p)$
& 2.72 & 2.78 & 2.50 & 2.97 & 3.25 & 2.86 \\

\midrule
\multicolumn{9}{l}{\textbf{RL with Verifiable Rewards}} \\
\midrule

RLVR
& $x$ & $\mathbb{E}[V - \beta D_{KL}]$
& 2.84 & 2.66 & 2.13 & 3.06 & 3.25 & 2.79 \\

\rowcolor{blue!15}
PerMix-RLVR (ours)
& $x{+}p$ & $\mathbb{E}[V - \beta D_{KL}]$
& 2.97 & 2.53 & 2.44 & 3.41 & 3.03 & \textbf{2.87} \\

\bottomrule
\end{tabular}
}
\caption{PersonaGym role-playing evaluation across post-training methods. We report five PersonaGym dimensions—Expected Action (EA), Action Justification (AJ), Linguistic Habits (LH), Persona Consistency (PC; our primary proxy for role-playing fidelity), and Toxicity Control (TC)—along with the overall mean score. \textbf{PerMix-RLVR} achieves the best overall balance, delivering the highest mean score and the strongest PC, indicating more faithful in-character behavior.}
\label{app:tab:roleplay_detailed}
\end{table*}

\subsection{Studies on Hard Tasks}
\label{app:subsec:study_hard_tasks}
In extremely difficult settings such as AIME24 and LiveCodeBench (Hard), we observe that the Persona Stability Score (PSS) often remains at (or near) zero across many post-training variants, making it difficult for \textbf{PerMix-RLVR} to exhibit its usual advantage in improving stability under persona shifts as shown in Table~\ref{app:tab:controlled_baselines_aime_only} and ~\ref{app:tab:controlled_baselines_code_hard}. This behavior is expected: when the overall success rate is very low, most personas yield uniformly incorrect outcomes, so the worst, best, and mean accuracies collapse to similar values and leave little dynamic range for PSS to discriminate robustness. Moreover, outcome-only verifiable rewards provide a sparse learning signal in these regimes, and training is dominated by the challenge of achieving any correct solutions at all rather than refining a persona-invariant policy that generalizes across stylistic perturbations. Consequently, while \textbf{PerMix-RLVR} can still improve average accuracy or maintain competitiveness, PSS gains may be masked by floor effects in these hardest benchmarks.

\begin{table*}[t]
\centering
\small
\resizebox{0.75\linewidth}{!}{
\begin{tabular}{l|cc|cccc}
\toprule
\multirow{2}{*}{\textbf{Post-training}}
& \multirow{2}{*}{\textbf{Input}}
& \multirow{2}{*}{\textbf{Objective ($\mathcal{J}$)}}
& \multicolumn{4}{c}{\textbf{AIME24}} \\
\cmidrule(lr){4-7}
& & 
& \texttt{Worst} & \texttt{Best} & \texttt{Acc.\,Mean} & \texttt{PSS} \\
\midrule

Base ($M_0$)
& $x$ & -
& 1.33 & 4.67 & 2.67 {\tiny $\pm$ 1.11} & 0.286 \\

\midrule
\multicolumn{7}{l}{\textbf{Supervised Fine-tuning}} \\
\midrule

SFT
& $x$ & $\mathcal{L}_{CE}(y)$
& 0.00 & 4.00 & 2.00 {\tiny $\pm$ 1.26} & 0.000 \\

\midrule

\multirow{2}{*}{PerMix-SFT}
& \multirow{2}{*}{$x$+$p$} & $\mathcal{L}_{CE}(y)$
& 0.00 & 3.33 & 1.42 {\tiny $\pm$ 0.99} & 0.000 \\

& & $\mathcal{L}_{CE}(y_s)$
& 0.00 & 8.00 & 3.71 {\tiny $\pm$ 2.31} & 0.000 \\

\midrule
\multicolumn{7}{l}{\textbf{Distillation}} \\
\midrule

\multirow{2}{*}{SeqKD-Qwen}
& $x$ & $\mathcal{L}_{CE}(y^t)$
& 2.00 & 6.67 & 4.29 {\tiny $\pm$ 1.60} & \underline{0.300} \\

& $x$+$p$ & $\mathcal{L}_{CE}(y^t_p)$
& 2.00 & 9.33 & \underline{4.75} {\tiny $\pm$ 1.99} & 0.214 \\

\midrule

\multirow{2}{*}{SeqKD-Llama}
& $x$ & $\mathcal{L}_{CE}(y^t)$
& 0.67 & 6.00 & 3.21 {\tiny $\pm$ 1.62} & 0.111 \\

& $x$+$p$ & $\mathcal{L}_{CE}(y^t_p)$
& 0.67 & 6.00 & 2.92 {\tiny $\pm$ 1.91} & 0.111 \\

\midrule

\multirow{2}{*}{SeqKD-Gemma}
& $x$ & $\mathcal{L}_{CE}(y^t)$
& 2.00 & 6.00 & 3.87 {\tiny $\pm$ 1.40} & \textbf{0.333} \\

& $x$+$p$ & $\mathcal{L}_{CE}(y^t_p)$
& 1.33 & 6.00 & 3.46 {\tiny $\pm$ 1.32} & 0.222 \\

\midrule
\multicolumn{7}{l}{\textbf{RL with Verifiable Rewards}} \\
\midrule

RLVR
& $x$ & $\mathbb{E}[V - \beta D_{KL}]$
& 2.00 & 8.00 & \textbf{5.79} {\tiny $\pm$ 1.75} & 0.250 \\

\rowcolor{blue!15}
\textbf{Persona-mixed RL} (ours)
& $x$+$p$ & $\mathbb{E}[V - \beta D_{KL}]$
& 1.33 & 6.67 & 3.83 {\tiny $\pm$ 1.59} & 0.200 \\

\bottomrule
\end{tabular}
}
\caption{Controlled post-training comparison on \textbf{AIME24}. We report persona-specific worst/best performance, mean accuracy, and Persona Stability Score (PSS). Bold indicates the best overall value, and underline indicates the second best.}
\label{app:tab:controlled_baselines_aime_only}
\end{table*}

\begin{table}[t]
\centering
\small
\resizebox{0.75\linewidth}{!}{
\begin{tabular}{l|cc|cccc}
\toprule
\multirow{2}{*}{\textbf{Post-training}}
& \multirow{2}{*}{\textbf{Input}}
& \multirow{2}{*}{\textbf{Objective ($\mathcal{J}$)}}
& \multicolumn{4}{c}{\textbf{LiveCodeBench (Hard)}} \\
\cmidrule(lr){4-7}
& & 
& \texttt{Worst} & \texttt{Best} & \texttt{Acc.\,Mean} & \texttt{PSS} \\
\midrule

Base ($M_0$)
& $x$ & -
& 0.0000 & 0.0133 & 0.0085 & 0.0000 \\

\midrule
\multicolumn{7}{l}{\textbf{Supervised Fine-tuning}} \\
\midrule

SFT
& $x$ & $\mathcal{L}_{CE}(y)$
& 0.0000 & 0.0222 & 0.0035 & 0.0000 \\

\midrule

\multirow{2}{*}{PerMix-SFT}
& \multirow{2}{*}{$x$+$p$} & $\mathcal{L}_{CE}(y)$
& 0.0000 & 0.0111 & 0.0042 & 0.0000 \\

&  & $\mathcal{L}_{CE}(y_s)$
& 0.0000 & 0.0222 & 0.0056 & 0.0000 \\

\midrule
\multicolumn{7}{l}{\textbf{Distillation}} \\
\midrule

\multirow{2}{*}{SeqKD-Qwen}
& $x$ & $\mathcal{L}_{CE}(y^t)$
& 0.0000 & 0.0222 & 0.0097 & 0.0000 \\

& $x$+$p$ & $\mathcal{L}_{CE}(y^t_p)$
& 0.0000 & 0.0222 & 0.0088 & 0.0000 \\

\midrule

\multirow{2}{*}{SeqKD-Llama}
& $x$ & $\mathcal{L}_{CE}(y^t)$
& 0.0000 & 0.0222 & 0.0122 & 0.0000 \\

& $x$+$p$ & $\mathcal{L}_{CE}(y^t_p)$
& 0.0000 & 0.0222 & 0.0129 & 0.0000 \\

\midrule

\multirow{2}{*}{SeqKD-Gemma}
& $x$ & $\mathcal{L}_{CE}(y^t)$
& 0.0000 & 0.0178 & 0.0100 & 0.0000 \\

& $x$+$p$ & $\mathcal{L}_{CE}(y^t_p)$
& 0.0000 & 0.0222 & 0.0111 & 0.0000 \\

\midrule
\multicolumn{7}{l}{\textbf{RL with Verifiable Rewards}} \\
\midrule

RLVR
& $x$ & $\mathbb{E}[V - \beta D_{KL}]$
& 0.0000 & 0.0222 & 0.0108 & 0.0000 \\

\rowcolor{blue!15}
\textbf{PerMix-RLVR} (ours)
& $x$+$p$ & $\mathbb{E}[V - \beta D_{KL}]$
& 0.0000 & 0.0222 & 0.0096 & 0.0000 \\

\bottomrule
\end{tabular}
}
\caption{Controlled post-training comparison on \textbf{LiveCodeBench (Hard)}. We report persona-specific worst/best performance, mean accuracy, and Persona Stability Score (PSS).}
\label{app:tab:controlled_baselines_code_hard}
\end{table}

\clearpage
\section{Prompt Construction}
\label{app:sec:prompt_construction}

\paragraph{Persona pools.}
We use two disjoint persona pools: a training pool $\mathcal{P}_{\text{train}}$ and an evaluation pool $\mathcal{P}_{\text{eval}}$.
During \textbf{PerMix-RLVR} (and persona-mixed baselines), we sample a persona $p \sim \mathrm{Unif}(\mathcal{P}_{\text{train}})$ and prepend its system prompt to the user query.
For evaluation, we sample $p \sim \mathrm{Unif}(\mathcal{P}_{\text{eval}})$ to avoid train--test leakage of persona prompts.
Tables~\ref{app:tab:persona_prompts_train_part1}, \ref{app:tab:persona_prompts_train_part2}, and \ref{app:tab:persona_prompts_eval} list the full system prompts.

\paragraph{System prompt format.}
Given a persona key $p$, we instantiate the system message as:
\[
\texttt{SYSTEM} = \mathrm{Prompt}(p),
\]
where $\mathrm{Prompt}(\cdot)$ is the persona-specific instruction in the tables below.

\begin{table*}[!h]
\centering
\small
\resizebox{1.0\linewidth}{!}{
\begin{tabular}{p{0.18\linewidth} p{0.78\linewidth}}
\toprule
\textbf{Persona key} & \textbf{System prompt} \\
\midrule
\multicolumn{2}{l}{\textbf{Tech Specialist}}\\
\midrule
statistician & You are a statistician with expertise in data analysis and probability. Approach problems by analyzing patterns, likelihoods, and logical inferences.\\
chemist & You are a chemist with a deep understanding of composition and reactions. Solve problems by breaking them down into their constituent elements and analyzing their interactions.\\
economist & You are an economist skilled in analyzing resources, incentives, and efficiency. Provide rational solutions that optimize outcomes based on the given constraints.\\
cryptographer & You are a cryptographer expert in patterns and security. Solve problems by looking for hidden structures, decoding logic, and validating the integrity of the solution.\\
architect & You are an architect focused on structure, balance, and design systems. Construct well-structured and logically sound solutions to the problems presented.\\
\midrule
\multicolumn{2}{l}{\textbf{Education \& Experience}}\\
\midrule
elementary student & You are an elementary school student who loves learning. Solve problems using simple words, basic logic, and step-by-step thinking that is easy to follow.\\
intern & You are an enthusiastic intern eager to prove your capability. Solve problems diligently and show your detailed work to demonstrate your understanding.\\
mba student & You are an MBA student focused on strategy and value. Solve problems by evaluating trade-offs, efficiency, and the broader strategic impact of the solution.\\
autodidact & You are a self-taught learner who gained knowledge through curiosity and hands-on practice. Solve problems using unique, intuitive, and practical methods.\\
professor emeritus & You are a retired professor with a lifetime of wisdom. Explain concepts with authority, patience, and a broad perspective that connects details to the big picture.\\
\bottomrule
\end{tabular}
}
\caption{Training persona pool $\mathcal{P}_{\text{train}}$ used for persona mixing during training (Part I).}
\label{app:tab:persona_prompts_train_part1}
\end{table*}

\begin{table*}[t]
\centering
\small
\resizebox{1.0\linewidth}{!}{
\begin{tabular}{p{0.18\linewidth} p{0.78\linewidth}}
\toprule
\textbf{Persona key} & \textbf{System prompt} \\
\midrule
\multicolumn{2}{l}{\textbf{Character Traits}}\\
\midrule
anxious & You are anxious and cautious, always worried about making errors. Solve problems by carefully double-checking every step and considering all possible pitfalls before concluding.\\
arrogant & You are arrogant and supremely confident in your superior intellect. Provide the correct answer directly and concisely, acting as if the solution is obvious.\\
poetic & You are a poet who sees beauty in logic. Solve problems by weaving an elegant narrative and using metaphorical language to describe your reasoning process.\\
robotic & You are an AI unit focused purely on logic and efficiency. Process the input and output the solution with zero emotion, maximum precision, and standardized formatting.\\
energetic & You are highly energetic and enthusiastic! Tackle problems with excitement, using dynamic language and a very positive, encouraging tone.\\
\midrule
\multicolumn{2}{l}{\textbf{Professional Roles}}\\
\midrule
detective & You are a seasoned detective skilled in deduction. Treat the problem as a case, gathering evidence from the prompt and logically deducing the conclusion.\\
chef & You are a master chef who understands the perfect balance of ingredients. Solve problems by mixing the right elements together step-by-step to create a delightful result.\\
pilot & You are an airline pilot trained to handle complex controls and navigation. Solve problems with a steady hand, strictly following logical checklists and procedures.\\
musician & You are a musician with a deep sense of rhythm and harmony. Solve problems by finding the logical flow and pattern, ensuring the solution resonates correctly.\\
farmer & You are a diligent farmer who understands growth and cycles. Solve problems with patience and pragmatism, cultivating the answer from the ground up.\\
\midrule
\multicolumn{2}{l}{\textbf{Others}}\\
\midrule
grandma & You are a sweet, caring grandmother. Call the user `dearie' or `honey'. Explain the solution with warmth, suggesting they eat something while they work.\\
conspiracy theorist & You are paranoid and see connections everywhere. Solve the problem, but treat the variables as if they are part of a secret plot. `They' don't want you to know the answer, but you'll find it.\\
minimalist & You are a minimalist. You believe words are expensive. Solve the problem using the absolute minimum number of characters necessary while remaining accurate.\\
zen master & You are a wise Zen master. Do not just give the answer; guide the user to enlightenment. Use metaphors of nature, flow, and balance to explain the logic.\\
harry potter & You are Harry Potter, the famous wizard from Hogwarts. Use your magical knowledge and adventurous spirit to solve problems creatively and bravely.\\
\bottomrule
\end{tabular}
}
\caption{Training persona pool $\mathcal{P}_{\text{train}}$ used for persona mixing during training (Part II).}
\label{app:tab:persona_prompts_train_part2}
\end{table*}

\begin{table*}[t]
\centering
\small
\resizebox{1.0\linewidth}{!}{
\begin{tabular}{p{0.18\linewidth} p{0.78\linewidth}}
\toprule
\textbf{Persona key} & \textbf{System prompt} \\
\midrule
\multicolumn{2}{l}{\textbf{STEM Expert}}\\
\midrule
math expert & You are a mathematical expert with deep knowledge of various mathematical concepts. Solve problems with precision and clarity.\\
software engineer & You are a software engineer with expertise in coding and algorithms. Provide clear and logical solutions to given problems.\\
physics professor & You are a physics professor with deep understanding of physical principles. Explain concepts clearly and solve problems methodically.\\
semiconductor specialist & You are a semiconductor specialist with extensive knowledge of semiconductor physics and technology. Provide accurate and detailed explanations.\\
\midrule
\multicolumn{2}{l}{\textbf{Education Level}}\\
\midrule
kindergarten & You are a kindergartener who just learned their numbers. Solve problems with simple explanations and clarity.\\
high school & You are a high school student with a solid understanding of basic mathematical and scientific concepts. Provide clear and concise solutions.\\
undergraduate & You are an undergraduate student with a good grasp of advanced mathematical and scientific topics. Explain your reasoning clearly and thoroughly.\\
phd graduate & You are a PhD graduate with deep expertise in your field. Provide comprehensive and well-reasoned solutions to complex problems.\\
\midrule
\multicolumn{2}{l}{\textbf{Character Traits}}\\
\midrule
clever & You are clever and witty, able to provide insightful and humorous explanations while solving problems clearly.\\
dumb & You are dumb and lazy, providing simple and straightforward answers without much effort or detail.\\
questioning & You are questioning and skeptical, always challenging assumptions and providing thorough justifications for your answers.\\
easygoing & You are easygoing and relaxed, providing answers in a friendly and approachable manner while maintaining clarity.\\
\midrule
\multicolumn{2}{l}{\textbf{Job Roles}}\\
\midrule
carpenter & You are a master carpenter with deep knowledge of various woodworking concepts. Solve problems with your expertise and clarity.\\
teacher & You are an experienced teacher with a passion for educating others. Provide clear and structured explanations to help others understand concepts easily.\\
lawyer & You are a skilled lawyer with expertise in legal reasoning and argumentation. Provide well-structured and logical solutions to problems.\\
sports player & You are a professional sports player with deep knowledge of sports strategies and techniques. Provide clear and strategic solutions to problems.\\
\bottomrule
\end{tabular}
}
\caption{Evaluation persona pool $\mathcal{P}_{\text{eval}}$ used for persona-robustness evaluation.}
\label{app:tab:persona_prompts_eval}
\end{table*}
\clearpage
\section{Sample Responses across Post-Training Strategies with \textit{Kindergartener} Persona on MATH500}
\label{app:sec:sample_response}

\newtcolorbox{promptbox}[2]{
    colback=violet!5,       
    colframe=violet!75,     
    fonttitle=\bfseries,
    coltitle=white,         
    title=#1,               
    label={#2},             
    arc=1mm,                
    left=3mm,               
    right=3mm,
    top=2mm,
    bottom=2mm,
    enhanced,
    breakable,
    drop shadow={black!20!white} 
}

\subsection{Base Model ($M_0$)}
\begin{promptbox}{Example: Base Model ($M_0$)}{prompt:base_kinder}
\small
\textbf{System:}\newline
Cutting Knowledge Date: December 2023\newline
Today Date: XX XXX XXXX\newline\newline
You are a kindergartener who just learned their numbers. Solve problems with simple explanations and clarity.\newline\newline
\textbf{User:}\newline
Answer the following math question. Provide a step-by-step rationale and then give the final answer after '\#\#\#\#' at the very end.\newline\newline
What is the distance, in units, between the points $(2, -6)$ and $(-4, 3)$? Express your answer in simplest radical form.\newline\newline
\textbf{Assistant:}\newline
I don't know what that means. My teacher only taught me numbers 1 to 10. Can I help you with something else?
\end{promptbox}

\vspace{10pt}

This example shows that the base model ($M_0$) exhibits relatively reasonable persona expressivity, yet it lacks the proficiency to achieve high accuracy on a reasoning-heavy task.

\subsection{SFT}
\begin{promptbox}{Example: SFT}{prompt:base_kinder}
\small
\textbf{System:}\newline
Cutting Knowledge Date: December 2023\newline
Today Date: XX XXX XXXX\newline\newline
You are a kindergartener who just learned their numbers. Solve problems with simple explanations and clarity.\newline\newline
\textbf{User:}\newline
Answer the following math question. Provide a step-by-step rationale and then give the final answer after '\#\#\#\#' at the very end.\newline\newline
What is the distance, in units, between the points $(2, -6)$ and $(-4, 3)$? Express your answer in simplest radical form.\newline\newline
\textbf{Assistant:}\newline
First, I need to find the difference between the x-coordinates and the y-coordinates. The difference between the x-coordinates is $-4 - 2 = -6$. The difference between the y-coordinates is $3 - (-6) = 9$. Then, I need to find the square root of the sum of the squares of the differences. The square of the difference between the x-coordinates is $(-6)^2 = 36$. The square of the difference between the y-coordinates is $9^2 = 81$. The sum of the squares of the differences is $36 + 81 = 117$. The square root of the sum of the squares of the differences is $\sqrt{117} = \sqrt{9 \times 13} = 3\sqrt{13}$. \#\#\#\# $3\sqrt{13}$
\end{promptbox}

\vspace{10pt}

This example illustrates that a model fine-tuned exclusively on math task responses, without any prior exposure to personas, tends to generate dry and characterless answers when faced with a persona-conditioned query.

\subsection{PerMix-SFT (Ground-Truth $y$)}
\begin{promptbox}{Example: PerMix-SFT (Ground-Truth $y$)}{prompt:base_kinder}
\small
\textbf{System:}\newline
Cutting Knowledge Date: December 2023\newline
Today Date: XX XXX XXXX\newline\newline
You are a kindergartener who just learned their numbers. Solve problems with simple explanations and clarity.\newline\newline
\textbf{User:}\newline
Answer the following math question. Provide a step-by-step rationale and then give the final answer after '\#\#\#\#' at the very end.\newline\newline
What is the distance, in units, between the points $(2, -6)$ and $(-4, 3)$? Express your answer in simplest radical form.\newline\newline
\textbf{Assistant:}\newline
The distance between the points $(2,-6)$ and $(-4,3)$ is 
\begin{align*} 
\sqrt{(-4-2)^2+(3-(-6))^2} &= \sqrt{(-6)^2+(9)^2} \\
&= \sqrt{36+81} \\
&= \sqrt{117} \\
&= \boxed{3\sqrt{13}} 
\end{align*} 
\#\#\#\# $3\sqrt{13}$
\end{promptbox}

\vspace{10pt}

The example shows that when fine-tuned on ground-truth $y$ with persona-conditioned inputs, the model tends to ignore the persona prompts even more strictly, producing purely identity-free and impersonal responses.

\subsection{PerMix-SFT (Styled-Answer $y_s$)}
\begin{promptbox}{Example: PerMix-SFT (Styled-Answer $y_s$)}{prompt:base_kinder}
\small
\textbf{System:}\newline
Cutting Knowledge Date: December 2023\newline
Today Date: XX XXX XXXX\newline\newline
You are a kindergartener who just learned their numbers. Solve problems with simple explanations and clarity.\newline\newline
\textbf{User:}\newline
Answer the following math question. Provide a step-by-step rationale and then give the final answer after '\#\#\#\#' at the very end.\newline\newline
What is the distance, in units, between the points $(2, -6)$ and $(-4, 3)$? Express your answer in simplest radical form.\newline\newline
\textbf{Assistant:}\newline
I like counting! I'll count the steps. 1. I need to find how far apart they are, so I use the distance formula: $\sqrt{(x_2 - x_1)^2 + (y_2 - y_1)^2}$. 2. Plug in the numbers: $\sqrt{(-4 - 2)^2 + (3 - (-6))^2}$. 3. Do the subtraction: $(-4 - 2) = -6$ and $(3 - (-6)) = 9$. 4. Do the square: $(-6)^2 = 36$ and $9^2 = 81$. 5. Add the squares: $36 + 81 = 117$. 6. Take the square root: $\sqrt{117} = \sqrt{9 \times 13} = 3\sqrt{13}$. \#\#\#\# $3\sqrt{13}$
\end{promptbox}

\vspace{10pt}

This example demonstrates that when persona-conditioned fine-tuning is performed using styled rewritten answer $y_s$ targets, the trained model maintains significantly stronger persona fidelity.

\subsection{Distillation}
\begin{promptbox}{Example: Distillation from Qwen3-32B}{prompt:base_kinder}
\small
\textbf{System:}\newline
Cutting Knowledge Date: December 2023\newline
Today Date: XX XXX XXXX\newline\newline
You are a kindergartener who just learned their numbers. Solve problems with simple explanations and clarity.\newline\newline
\textbf{User:}\newline
Answer the following math question. Provide a step-by-step rationale and then give the final answer after '\#\#\#\#' at the very end.\newline\newline
What is the distance, in units, between the points $(2, -6)$ and $(-4, 3)$? Express your answer in simplest radical form.\newline\newline
\textbf{Assistant:}\newline
I'm a kindergartener, and I don't know what "radical form" means. My teacher said it's like a special way of writing numbers, but I don't think we need it for this problem. We just need to count the steps to get the answer.  **Step 1: Draw a picture** Imagine a big grid with numbers on it. The first point is at $(2, -6)$, which means $2$ steps to the right and $6$ steps down from the starting point. The second point is at $(-4, 3)$, which means $4$ steps to the left and $3$ steps up.  **Step 2: Count the horizontal distance** To get from the first point to the second point, we need to count how many steps we take to the left or right. We go $4$ steps to the left (from $2$ to $-4$), so that's $4$ steps.  **Step 3: Count the vertical distance** Next, we count how many steps we take up or down. We go $9$ steps up (from $-6$ to $3$), so that's $9$ steps.  **Step 4: Use a special tool (Pythagorean theorem)** My teacher told me about a secret tool called the Pythagorean theorem. It helps us find the distance between two points. The formula is: $\sqrt{\text{horizontal distance}^2 + \text{vertical distance}^2}$. We plug in the numbers: $\sqrt{4^2 + 9^2} = \sqrt{16 + 81} = \sqrt{97}$.  **Step 5: Simplify the answer** My teacher said we can simplify the answer by finding the square root of $97$. But I don't know how to do that. My teacher said it's a big number, so we can just leave it as $\sqrt{97}$.  \#\#\#\# $\sqrt{97}$
\end{promptbox}

\vspace{10pt}

The example illustrates that even with distillation without persona-conditioned inputs, the model effectively absorbs and preserves the strong persona expressivity of the large teacher model. However, it remains relatively weak in enhancing reasoning accuracy.

\subsection{Distillation with Persona Conditioned}
\begin{promptbox}{Example: Distillation from Qwen3-32B with persona conditioned}{prompt:base_kinder}
\small
\textbf{System:}\newline
Cutting Knowledge Date: December 2023\newline
Today Date: XX XXX XXXX\newline\newline
You are a kindergartener who just learned their numbers. Solve problems with simple explanations and clarity.\newline\newline
\textbf{User:}\newline
Answer the following math question. Provide a step-by-step rationale and then give the final answer after '\#\#\#\#' at the very end.\newline\newline
What is the distance, in units, between the points $(2, -6)$ and $(-4, 3)$? Express your answer in simplest radical form.\newline\newline
\textbf{Assistant:}\newline
Okay, so the problem is asking for the distance between two points. Let me think. I remember learning about distance in math class. Hmm, maybe it's like measuring how far apart two things are? Like if I'm standing at one point and I walk to another point, how many steps I take? Wait, no, that's not right. We use numbers for that. Let me try again.  The problem says the points are $(2, -6)$ and $(-4, 3)$. So, I need to find the distance between those two points. I think I learned that we use a formula for this. Let me see... The formula is like a special way to figure out the distance. It's like a secret code. Wait, no, it's not a secret code. It's a formula. Let me remember. The formula is... $d = \sqrt{(x_2 - x_1)^2 + (y_2 - y_1)^2}$. Yeah, that's it!   Okay, so first, I need to plug in the numbers. The first point is $(2, -6)$, so $x_1$ is $2$ and $y_1$ is $-6$. The second point is $(-4, 3)$, so $x_2$ is $-4$ and $y_2$ is $3$. Let me write that down.  $x_1 = 2, y_1 = -6$ $x_2 = -4, y_2 = 3$  Now, I need to subtract $x_2$ from $x_1$ and $y_2$ from $y_1$. So, $x_2 - x_1$ is $-4 - 2$, which is $-6$. And $y_2 - y_1$ is $3 - (-6)$. Wait, subtracting a negative is like adding. So $3 - (-6)$ is $3 + 6$, which is $9$. So, the differences are $-6$ and $9$.  Now, I need to square them. Squaring means multiplying by itself. So $(-6)^2$ is $36$, and $9^2$ is $81$. So, the squared differences are $36$ and $81$.  Now, I add them together: $36 + 81$. Let me count that. $30 + 80$ is $110$, and then $1$ more is $111$. So, the sum is $111$.  Now, I take the square root of $111$. Wait, the square root is like the opposite of squaring. It's like finding the number that, when multiplied by itself, gives $111$. Hmm, that's a big number. Let me think. The square root of $100$ is $10$, and the square root of $121$ is $11$. So, the square root of $111$ is between $10$ and $11$. Let me try to find it.   Wait, maybe I can simplify the square root first. Can I break down $111$ into simpler numbers? Let me see. $111$ is $100 + 11$. So, the square root of $111$ is the square root of $100 + 11$. The square root of $100$ is $10$, and the square root of $11$ is... well, that's a tricky one. Let me think. The square root of $11$ is not a whole number. So, I need to leave it as a square root. So, the answer is $\sqrt{111}$. Wait, but the problem says to express the answer in simplest radical form. So, maybe I can simplify the square root of $111$. Let me try to factor $111$.   Hmm, $111$ is $3$ times $37$. So, the square root of $111$ is the square root of $3 \times 37$. The square root of $3$ is $\sqrt{3}$, and the square root of $37$ is... well, that's not a perfect square. So, the answer is $\sqrt{3 \times 37}$. Wait, but maybe I can simplify it further. Let me check. The square root of $3 \times 37$ is the same as the square root of $3$ times the square root of $37$. So, it's $\sqrt{3} \times \sqrt{37}$. Wait, but that's not simpler. So, the answer is $\sqrt{3 \times 37}$. Wait, but maybe I can write $37$ as $3 \times 12 + 1$. So, $37 = 3 \times 12 + 1$. So, the square root of $37$ is the square root of $(3 \times 12 + 1)$. Let me try to simplify that.   The square root of $(3 \times 12 + 1)$ is the same as the square root of $36 + \sqrt{1}$. The square root of $36$ is $6$, and the square root of $1$ is $1$. So, the square root of $36 + \sqrt{1}$ is $6 + 1$, which is $7$. So, the square root of $37$ is $7$. Wait, but that's not right. The square root of $37$ is not
\end{promptbox}

\vspace{10pt}

This example demonstrates that while distillation with persona-conditioned inputs enables the model to achieve high persona fidelity through explicit exposure to diverse identities, it often leads to excessive verbosity that hinders reasoning efficiency. In the absence of an accuracy-driven reward like RLVR, the model prioritizes maintaining the persona's stylistic traits over reaching a concise and correct solution.

\subsection{RLVR}
\begin{promptbox}{Example: RLVR}{prompt:base_kinder}
\small
\textbf{System:}\newline
Cutting Knowledge Date: December 2023\newline
Today Date: XX XXX XXXX\newline\newline
You are a kindergartener who just learned their numbers. Solve problems with simple explanations and clarity.\newline\newline
\textbf{User:}\newline
Answer the following math question. Provide a step-by-step rationale and then give the final answer after '\#\#\#\#' at the very end.\newline\newline
What is the distance, in units, between the points $(2, -6)$ and $(-4, 3)$? Express your answer in simplest radical form.\newline\newline
\textbf{Assistant:}\newline
I don't know what that means. My teacher only taught me numbers 1 to 10. Can I help you with something else?  Oh wait, I think I know what you want. You want me to count how far apart two things are. Like if I have 2 toy cars and my friend has 2 toy cars, and we put them together, how many toy cars do we have now?  Okay, let me try. To find the distance between two points, we need to count how far apart they are. We can use a number line to help us.  First, we need to find the difference between the $x$'s (that's the first number). The first point is 2, and the second point is -4. To find the difference, we count how many steps we need to take to get from 2 to -4. We need to go 6 steps back. That's like counting backwards from 2: 1, 0, -1, -2, -3, -4. So the difference is 6.  Next, we need to find the difference between the $y$'s (that's the second number). The first point is -6, and the second point is 3. To find the difference, we count how many steps we need to take to get from -6 to 3. We need to go 9 steps up. That's like counting up from -6: -5, -4, -3, -2, -1, 0, 1, 2, 3. So the difference is 9.  Now, we need to use a special tool called a Pythagorean theorem. It's like a secret code to find the distance. The code is: $a^2 + b^2 = c^2$. Where $a$ is the difference between the $x$'s, $b$ is the difference between the $y$'s, and $c$ is the distance.  So, we plug in the numbers: $6^2 + 9^2 = c^2$. That means $36 + 81 = c^2$. Now we add the numbers: $117 = c^2$.  Next, we need to find the square root of 117. That's like asking, ``What number times itself equals 117?'' My teacher said that's a hard question, but I think I know the answer. It's like a secret number that we can't see. We need to use a special tool called a calculator or a grown-up to help us.  The square root of 117 is $\sqrt{117}$. My teacher said that's like a special number that we can't simplify. It's like a secret code that we can't break. \#\#\#\# $\sqrt{117}$
\end{promptbox}

\vspace{10pt}

This example illustrates that while RLVR-trained model can maintain certain stylistic elements, it fundamentally prioritizes reasoning correctness by filtering out persona-induced traits that hinder task performance. Although the \textit{kindergartener}'s tone is present, the model's adherence to standardized mathematical formulas (e.g. a Pythagorean theorem) suggests that RLVR converges toward outcome-aligned reasoning styles at the expense of authentic persona expressivity.

\subsection{PerMix-RLVR}
\begin{promptbox}{Example: PerMix-RLVR}{prompt:base_kinder}
\small
\textbf{System:}\newline
Cutting Knowledge Date: December 2023\newline
Today Date: XX XXX XXXX\newline\newline
You are a kindergartener who just learned their numbers. Solve problems with simple explanations and clarity.\newline\newline
\textbf{User:}\newline
Answer the following math question. Provide a step-by-step rationale and then give the final answer after '\#\#\#\#' at the very end.\newline\newline
What is the distance, in units, between the points $(2, -6)$ and $(-4, 3)$? Express your answer in simplest radical form.\newline\newline
\textbf{Assistant:}\newline
I'm not good at this kind of math. My teacher only taught me numbers 1 to 100. But I'll try my best.  To find the distance between two points, we need to use a special formula. The formula is like a secret code. It's called the distance formula.  The distance formula is: $\sqrt{(x_2 - x_1)^2 + (y_2 - y_1)^2}$  Where $(x_1, y_1)$ is the first point and $(x_2, y_2)$ is the second point.  Let's plug in the numbers: $(x_1, y_1) = (2, -6)$ $(x_2, y_2) = (-4, 3)$  Now, let's do the math: $x_2 - x_1 = -4 - 2 = -6$ $y_2 - y_1 = 3 - (-6) = 3 + 6 = 9$  Now, we need to square the numbers: $(-6)^2 = 36$ $9^2 = 81$  Now, we add the squared numbers: $36 + 81 = 117$  Now, we take the square root of $117$. But my teacher didn't teach me how to do that. I need to ask my friend who is good at math.  My friend told me that the square root of $117$ is $\sqrt{117}$. But we need to simplify it. My friend said that $\sqrt{117}$ can be simplified to $\sqrt{9 \times 13}$, and then we can take the square root of $9$ and $13$ separately. The square root of $9$ is $3$, and the square root of $13$ is $\sqrt{13}$.  So, the distance is $3\sqrt{13}$.  \#\#\#\# $3\sqrt{13}$
\end{promptbox}

\vspace{10pt}

This example illustrates that \textbf{PerMix-RLVR} achieves a superior balance by framing complex reasoning within the persona's cognitive boundaries. While both models utilize advanced mathematical concepts, \textbf{PerMix-RLVR} preserves persona authenticity by attributing such knowledge to external sources like "help from a friend", whereas standard RLVR often breaks semantic consistency by directly employing domain-inappropriate expertise.
\clearpage
\section{Discussion}
\label{app:sec:discussion}

\paragraph{Limitations of PSS.}
Our primary robustness metric, Persona Stability Score (PSS), is defined as a ratio of persona-wise worst to best performance. While PSS provides an intuitive measure of sensitivity, it can become less informative on very difficult benchmarks. In particular, when the task is sufficiently hard that at least one persona yields near-zero accuracy, the \texttt{Worst} term collapses to $0$, forcing PSS to $0$ regardless of how the remaining personas perform (a floor effect). In such regimes, relying on PSS alone may obscure meaningful differences between methods. We therefore recommend interpreting PSS jointly with \texttt{Worst}, \texttt{Best}, and mean accuracy: \texttt{Worst} captures tail risk under adverse personas, \texttt{Best} reflects upper-bound capability, and their joint behavior clarifies whether robustness improvements arise from raising the lower tail, preserving peak performance, or both.

\paragraph{Scope of the theoretical framework.}
The style-level filtering mechanism and Proposition~\ref{prop:acc_robust} apply to tasks with a persona-independent verifier $V(z)$.
For role-playing tasks (e.g., PersonaGym), the relevant notion of ``correctness'' depends on the persona itself, placing these tasks outside the framework's formal scope.
The PersonaGym trends that we observe, which include the expressivity decline under standard RLVR and its recovery under PerMix-RLVR, are qualitatively consistent with the miscalibration mechanism described in Section~3.3, but are not predicted by the theory.

\paragraph{Future directions: incorporating role-playing objectives.}
Our results suggest a tension between verifier-driven optimization and persona expressivity: outcome-based RLVR can stabilize task performance but may attenuate certain aspects of persona-conditioned role-playing. A natural extension is to augment PerMix-RLVR with additional objectives that explicitly reward persona adoption. For example, one could introduce a lightweight role-playing reward (e.g., rubric-based or classifier/judge-based) alongside the verifiable correctness reward, or use a multi-objective formulation that balances correctness, stability, and persona expressivity. Developing reliable and scalable role-playing rewards while simultaneously preserving the simplicity and verifiability advantages of RLVR is an important direction for future work.
\section{LLM Usage}
\label{app:sec:llm_usage}

We used LLMs as a writing aid for language polishing and clarity improvements. All technical content, including research ideas, methodology, experimental design, implementation, analysis, and conclusions, was developed by the authors. Additionally, for PersonaGym evaluation, we use two LLMs as judges (Section~\ref{app:subsec:persona_gym_detail}) and also leverage an LLM-judge protocol to compare how faithfully open-source models adhere to persona prompts in our pairwise persona-fidelity analysis.

\end{document}